\documentclass{article}

\usepackage{iclr2027_conference,times}
\usepackage{amsmath,amsfonts,bm}

\def\eqref#1{equation~\ref{#1}}
\def\1{\bm{1}}

\DeclareMathAlphabet{\mathsfit}{\encodingdefault}{\sfdefault}{m}{sl}
\SetMathAlphabet{\mathsfit}{bold}{\encodingdefault}{\sfdefault}{bx}{n}

\newcommand{\boldx}{\boldsymbol{x}}

\newcommand{\target}{y}

\newcommand{\lossfunc}{\ell}
\newcommand{\trainloss}{\mathcal{L}}
\newcommand{\boldw}{\boldsymbol{w}}

\newcommand{\boldv}{\boldsymbol{v}}
\newcommand{\boldz}{\boldsymbol{z}}
\newcommand{\boldu}{\boldsymbol{u}}

\newcommand{\boldzero}{\boldsymbol{0}}

\newcommand{\calS}{\mathcal{S}}

\newcommand{\boldb}{\boldsymbol{b}}

\newcommand{\boldU}{\boldsymbol{U}}

\newcommand{\boldLambda}{\boldsymbol{\Lambda}}
\newcommand{\boldSigma}{\boldsymbol{\Sigma}}

\usepackage{amsmath,amsthm}
\newtheorem{theorem}{Theorem}
\newtheorem*{theorem*}{Theorem}
\newtheorem{lemma}{Lemma}
\newtheorem{proposition}[lemma]{Proposition}
\newtheorem{definition}{Definition}
\usepackage{tabularx}
\usepackage{enumitem}
\usepackage[utf8]{inputenc}
\usepackage[T1]{fontenc}
\usepackage[english]{babel}
\usepackage[babel=true]{microtype}
\usepackage{graphicx}
\usepackage{csquotes}
\usepackage{makecell}
\usepackage{subcaption}
\usepackage{bm}
\usepackage{url}
\usepackage{hyperref}
\usepackage{float}
\usepackage{array,multirow}
\usepackage[flushleft]{threeparttable}
\usepackage{booktabs}
\hypersetup{
    colorlinks=true,
    linkcolor=blue,
    citecolor=blue,
    filecolor=magenta,
    urlcolor=blue,
    pdftitle={Whitening Improves Robustness to Spurious Correlations in Linear Probes},
    pdfpagemode=FullScreen,
}

\title{Whitening Improves Robustness to Spurious Correlations in Linear Probes}
\author{
Floris Holstege$^{1,2}$,
Bram Wouters$^{1}$,
Noud van Giersbergen$^{1}$,
Cees Diks$^{1,2}$ \\
$^{1}$University of Amsterdam, Department of Quantitative Economics
\hspace{0.2cm}
$^{2}$Tinbergen Institute \\
\texttt{\{f.g.holstege,b.m.wouters,n.p.a.vangiersbergen,c.g.h.diks\}@uva.nl}
\vspace{-0.2cm}
}
\iclrfinalcopy
\begin{document}

\maketitle

\begin{abstract}
Deep neural networks tend to rely on simple features that may be spurious and thus fail to generalize. 
We study this problem in the setting of linear probes, where a (generalized) linear model is fitted on the representations of a (pretrained) model. 
We use the connection of these models to the max-margin classifier, and show they favor directions associated with large eigenvalues of the covariance matrix. 
Whitening removes this preference by equalizing the eigenvalues of the covariance matrix. 
This observation motivates whitening as a preprocessing step that can reduce reliance on spurious correlations without requiring prior knowledge of their presence or labeled data. 
We examine the effect of whitening on a synthetic data-generating process and standard spurious correlation benchmarks, and find that it improves robustness. 
We also find that whitening can improve robustness when added to existing approaches.
\end{abstract}

\section{Introduction}
Deep neural networks (DNNs) have shown a tendency to use `simple' features for prediction, a phenomenon described as the \emph{simplicity bias} \citep{arpit2017closerlookmemorizationdeep, pmlr-v97-rahaman19a, Kalimeris2019IncreasingComplexity, Morwani}.
Previous work argues that the simplicity bias of DNNs contributes to them relying on so-called \emph{spurious correlations} \citep{geirhos_shortcut, vasudeva2023mitigatingsimplicitybiasdeep}.
For example, models can achieve strong performance by exploiting a simple feature (e.g., background/texture in images, \citet{geirhos2022imagenettrainedcnnsbiasedtexture}), yet suffer a large drop in performance when this cue is absent \citep{sagawa2020distributionallyrobustneuralnetworks, pmlr-v139-koh21a, Xiao2021Backgrounds}.

In this work, we focus on \textit{linear probes}, (generalized) linear models fitted on a representation generated by a DNN. %
Linear probes are widely used for transfer learning \citep{Donahue2014DeCAF, Azizpour2016Transferability, Kornblith2019Transfer, Bar2024FrozenFeatures}, and have been shown to be sufficient to address spurious correlations \citep{kirichenko2023layerretrainingsufficientrobustness, izmailov2022featurelearningpresencespurious,labonte2023lastlayerretraininggrouprobustness, qiu2024complexitymattersdynamicsfeature, hill2025unreasonableeffectivenesslastlayerretraining}.
We start by analyzing the inclination of the max-margin classifier to rely on \textit{simple} features, defined as features associated with large eigenvalues of the covariance matrix \citep{Maennel2020RandomLabels, pezeshki2021gradientstarvationlearningproclivity, Baratin2021FeatureAlignment, bombari2025spuriouscorrelationshighdimensional}.
The max-margin classifier is relevant, as it explains the behavior of $\ell_2$-regularized logistic regression on separable data when regularization approaches zero \citep{Rosset_2003} or when trained with gradient descent \citep{soudry2018implicit,Nacson2019Separable}. It also explains some types of DNNs (see \citet{gunasekar, lyu2020gradientdescentmaximizesmargin, Chizat2020ImplicitBias}).
Our analysis builds upon earlier work on the behavior of the max-margin classifier in the presence of spurious correlations \citep{pmlr-v119-sagawa20a, Kini2021GroupSensitive, puli_2023, nagarajan2024understandingfailuremodesoutofdistribution}.

We then show that whitening the data (centering, scaling to unit variance, and removing correlations) can remove the tendency of the max-margin classifier to rely on simple features.
This suggests a simple strategy for reducing the tendency of linear probes to rely on spurious correlations: applying a whitening transformation before fitting.
We show empirically that whitening can yield substantial gains in terms of robustness against spurious correlations. 
Importantly, this requires \emph{no pre-existing knowledge of the existence or nature of the spurious correlation} nor additional labels, and is computationally light.
We find that whitening can also be used to improve existing approaches that address spurious correlations such as Deep Feature Reweighting \citep{kirichenko2023layerretrainingsufficientrobustness}.

This paper is structured as follows. In Section~\ref{sec:related}, we briefly discuss related work.
In Section~\ref{sec:theory}, we show the max-margin classifier has a simplicity bias that can be addressed with whitening.
Section~\ref{sec:synthetic-dgp} illustrates the effect of whitening on a synthetic data-generating process from the literature on spurious correlations. 
In Section~\ref{sec:experiments} we evaluate the effect of whitening on linear probes for several spurious correlation benchmarks. 
Finally, we discuss limitations and conclude in Section~\ref{sec:conclusion}.

\section{Related work}
\label{sec:related}

\textbf{Spurious correlations \& simplicity bias in DNNs}: The problem of DNNs relying on spurious correlations is widespread across different architectures, modalities and datasets \citep{Lapuschkin_2019, pmlr-v139-koh21a}.
 One frequently mentioned cause for this reliance on spurious correlations is that DNNs tend to focus on `simple features' \citep{shah_2020}. 
 While many definitions of simple features exist, they commonly relate to the spectrum of the covariance matrix \citep{pezeshki2021gradientstarvationlearningproclivity, bombari2025spuriouscorrelationshighdimensional}. 
 We adopt a similar definition of simplicity.

\textbf{Approaches for addressing spurious correlations}: A wide range of approaches have been suggested for addressing spurious correlations (see \citet{ye2025cleverhansmiragecomprehensive} for a full overview). 
 The seminal work of \citet{kirichenko2023layerretrainingsufficientrobustness} suggests that in order to address spurious correlations, it is sufficient to focus on retraining the last layer of a DNN. 
 While sometimes insufficient \citep{le2024layerretrainingtrulysufficient, ye2025freezetrainprovablerepresentation}, approaches addressing spurious correlations at the last layer have proven remarkably effective \citep{izmailov2022featurelearningpresencespurious, hill2025unreasonableeffectivenesslastlayerretraining}.
However, most of these approaches rely on pre-existing knowledge of the spurious correlation, as well as auxiliary labels associated with the spurious feature. 
While some approaches only rely on such labels for a validation set (see for example \citet{pmlr-v139-liu21f} or \citet{qiu2023simplefastgrouprobustness}), only a few methods have been suggested that work without any auxiliary labels for both training and validation (see, for example, \citet{bayat2024pitfallsmemorizationmemorizationhurts} or \citet{le2024spuriousityimprovingrobustnessspurious}).
 Our suggested approach falls in this latter category.

\textbf{Whitening in DNNs}: Whitening-based normalization layers have been used in order to stabilize and speed up the training of DNNs \citep{huang2018decorrelatedbatchnormalization, huang2019iterativenormalizationstandardizationefficient}, and even improve their robustness \citep{gao-etal-2022-kernel}.
 In the field of domain adaptation, whitening has been used to match source and target feature statistics \citep{sun2016correlationalignmentunsuperviseddomain, roy2020unsuperviseddomainadaptationusing}. 
More recently, \citet{cho2025controllablefeaturewhiteninghyperparameterfree} use whitening in combination with auxiliary labels to enforce fairness constraints. 
It has also been used to distinguish different concepts in the representations of neural networks \citep{Chen_2020}.
\citet{wadia2021whiteningsecondorderoptimization} find that whitening can hurt generalization as it removes information about the second moments of the data. In contrast, we show this information favors simple features, so removing it through whitening can improve out-of-distribution robustness.
A similar finding was reported by \citet{ghanooni2025mitigating}, who show that training semi-supervised DNNs to have representations with uniform eigenvalues improves generalization in linear probes. 
Our work shows that this effect can also be achieved by whitening the representation rather than (re)training the DNN.

\section{Simplicity bias of the max-margin classifier}
\label{sec:theory}

Let $\{(\boldx_i,\target_i)\}_{i=1}^n$ be a dataset with
$\target_i\in\{-1,1\}$ and $\boldx_i\in\mathbb{R}^p$.
Throughout, we assume the data is centered, $\sum_{i=1}^n \boldx_i = \boldzero$.
We consider the standard setup of a logistic loss with coefficients $\boldw\in\mathbb{R}^p$ and $\ell_2$ regularization (excluding the intercept):
\begin{align}
\trainloss(\boldw)
\;:=\;
\frac{1}{n}\sum_{i=1}^n \lossfunc\!\left(\target_i\,\boldw^\top \boldx_i\right) + \tau \|\boldw \|_2^2,
\qquad
\lossfunc(u):=\log(1+e^{-u}). \label{eq:log_loss}
\end{align}
To simplify our analysis, we assume the data is linearly separable and first focus on the case where we optimize $\trainloss(\boldw)$ in the limit $\tau \rightarrow 0$.
In that case, the direction of the minimizer converges to the direction of the max-margin classifier \citep{Rosset_2003}, which is given by
\begin{align}
\hat{\boldw}_{\mathrm{mm}}
=
\arg\min_{\boldw\in\mathbb{R}^p}\|\boldw\|_2^2
\quad
\text{s.t.}
\quad
\target_i\,\boldw^\top\boldx_i\ge 1,\ \forall i.
\label{eq:mm_original}
\end{align}
Note that the max-margin classifier also arises when we minimize $\trainloss(\boldw)$ with $\tau=0$ on linearly separable data using gradient descent (GD) \citep{soudry2018implicit}.
We consider the case of a non-zero $\tau$ in Section \ref{sec:l2}.

Many definitions of `simple' features focus on the spectrum of the empirical covariance matrix. This matrix and its eigendecomposition are given by
\begin{align}
\hat{\boldSigma}
\;:=\;
\frac1{n-1}\sum_{i=1}^n \boldx_i\boldx_i^\top
\;=\;
\boldU\hat{\boldLambda}\boldU^\top,
\end{align}
with empirical eigenvalues
$\hat{\lambda}_1\ge\cdots\ge\hat{\lambda}_r>0$ and
$\hat{\lambda}_{r+1}=\cdots=\hat{\lambda}_p=0$,
where $r=\operatorname{rank}(\hat{\boldSigma})$, and $\boldU$ is an orthonormal matrix.
To quantify the extent to which a classifier relies on simple features, we use the Rayleigh quotient.
\begin{definition}[Rayleigh quotient \& simplicity]
\label{def:simplicity}
For a classifier $\boldw\in\mathbb{R}^p$, its Rayleigh quotient is 
\begin{align}
\mathrm{rq}(\boldw)
\;:=\;
\frac{\boldw^\top \hat{\boldSigma}\,\boldw}{\boldw^{\top}\boldw}
\;\in\;[0,\hat{\lambda}_1].
\end{align}
We call $\boldw$ simpler than $\boldw'$ when $\mathrm{rq}(\boldw) > \mathrm{rq}(\boldw')$.
\end{definition}
Using the Rayleigh quotient to quantify simplicity is inspired by previous work that argues in GD, learning is dominated by eigenvectors of $\hat{\boldSigma}$ with large eigenvalues, as the eigenvalues accelerate learning along those directions \citep{pezeshki2021gradientstarvationlearningproclivity, ghanooni2025mitigating}.

We claim that the max-margin classifier favors simple features according to Definition~\ref{def:simplicity}. To show this, define the observations $\boldx_i$ and classifier $\boldw$ in whitened coordinates, 
\begin{align}
\boldz_i := (\hat{\boldLambda}^\dagger)^{1/2}\boldU^\top\boldx_i,
\qquad
\boldv := \hat{\boldLambda}^{1/2}\boldU^\top\boldw.
\label{eq:whiten-map}
\end{align}

\noindent Here, $\dagger$ denotes the Moore--Penrose inverse. The max-margin classifier of \eqref{eq:mm_original} can be rewritten in the whitened coordinates as
\begin{align}
\hat{\boldw}_{\mathrm{mm}}=\boldU(\hat{\boldLambda}^\dagger)^{1/2}\hat{\boldv}_{\mathrm{mm}}, \qquad \text{with} \qquad
\hat{\boldv}_{\mathrm{mm}}=
\arg\min_{\boldv\in \mathcal{S}_z}
\boldv^\top\hat{\boldLambda}^\dagger \boldv
\quad \text{s.t.}\quad
\target_i\, \boldv^\top \boldz_i \ge 1,\ \forall i,
\label{eq:mm_whitened}
\end{align}
with $\mathcal{S}_z := \operatorname{span}\{\boldz_1,\dots,\boldz_n\}$. To see this equivalence, let $\calS:=\operatorname{span}\{\boldx_1,\ldots,\boldx_n\}=\operatorname{range}(\hat{\boldSigma})$ and decompose any $\boldw=\boldw_{\calS}+\boldw_\perp$, where $\boldw_\perp$ is in the null space of $\calS$. Because every $\boldx_i\in\calS$, replacing  $\boldw$ by $\boldw_{\calS}$ preserves the margins $\target_i\boldw^{\top}\boldx_i$ and does not increase the norm, so the minimizer of \eqref{eq:mm_original} belongs to $\calS$. On $\calS$, the change of variables is invertible and satisfies
\begin{align}
\boldx_i
=\boldU\hat{\boldLambda}^{1/2}\boldz_i,\quad
\boldw
=\boldU(\hat{\boldLambda}^\dagger)^{1/2}\boldv,\quad
\boldw^\top\boldx_i
=\boldv^\top\boldz_i,\quad
\|\boldw\|_2^2
=\boldv^\top\hat{\boldLambda}^\dagger\boldv,
\end{align}
which maps the constraints and objective in \eqref{eq:mm_original} exactly to those in \eqref{eq:mm_whitened}.

By writing the max-margin classifier in whitened coordinates, it becomes clear that it exhibits a preference for simplicity as defined in Definition~\ref{def:simplicity}.
In particular, we can write the minimum-norm objective in \eqref{eq:mm_whitened} as the weighted norm:
\begin{align}
    \boldv^\top\hat{\boldLambda}^\dagger \boldv
    = \sum^r_{k=1}\frac{v_k^2}{\hat{\lambda}_k}.
\end{align}
This representation reveals that components of $\boldv$ along directions with large empirical eigenvalues are weighted to a lesser extent, since they are divided by $\hat{\lambda}_k$. 
Thus, the max-margin objective tends toward features that are associated with large eigenvalues of the empirical covariance matrix, making the classifier simpler according to Definition~\ref{def:simplicity}.
We refer to this as the \textit{simplicity bias} of the max-margin classifier. 
This bias means that among separating classifiers, the objective can prefer classifiers that rely more on simpler features, even when those provide less favorable margins in the whitened representation.
We argue that this is undesirable in the context of spurious correlations,
based on the idea that these tend to align with simple features \citep{Teney2022EvadingSimplicity, pmlr-v238-yang24c}. 

The simplicity bias can be removed by whitening, which sets all nonzero empirical eigenvalues to one. Formally, this is achieved by transforming the data to whitened coordinates as in \eqref{eq:whiten-map}.
The max-margin classifier fitted on the
whitened data is given by
\begin{align}
\hat{\boldv}_{\mathrm{mm},\mathrm{whiten}} = \arg\min_{\boldv\in \mathcal{S}_z}
\|\boldv\|_2^2
\quad \text{s.t.}\quad
\target_i\, \boldv^\top \boldz_i \ge 1,\ \forall i.
\label{eq:mm-whitened-data}
\end{align}
The objective is now the standard $\ell_2$ norm and no longer weights a
direction according to its empirical eigenvalue. Consequently, the remaining
preference among separating directions is determined by their margins in the
whitened representation.

\subsection{The role of \texorpdfstring{$\ell_2$}{L2} regularization}
\label{sec:l2}
When the $\ell_2$ regularization penalty $\tau$ does not converge to 0,
the classifier does not necessarily converge to the max-margin classifier. However, we observe a similar simplicity bias for a fixed $\ell_2$ penalty $\tau > 0$.
For centered data and $\boldw \in \calS$, $ \boldv \in \calS_{\boldz}$, the change of variables in \eqref{eq:whiten-map} gives
\begin{align}
\target_i\,\boldw^\top\boldx_i = \target_i\,\boldv^\top\boldz_i
\qquad\text{and}\qquad
\|\boldw \|_2^2 = \boldv^{\top} \boldLambda^\dagger \boldv.
\label{eq:ridge-identities}
\end{align}
We can use this to write the loss function of \eqref{eq:log_loss} in whitened coordinates:
\begin{align}
\trainloss(\boldv)
=
\frac{1}{n}\sum_{i=1}^n
\lossfunc\left(\target_i\boldv^\top \boldz_i\right)
+ \tau \sum_{j=1}^r \frac{v_j^2}{\lambda_j}.
\end{align}
Hence, coefficients along directions associated with small eigenvalues receive a larger penalty.
$\ell_2$ regularization induces the same simplicity bias identified in the previous section.
Again, we can resolve this with whitening, as then the $\ell_2$-regularized loss becomes $\frac{1}{n}\sum_{i=1}^n
\lossfunc\left(\target_i\boldv^\top \boldz_i\right)
+ \tau \|\boldv \|_2^2$. We use a similar argument to show the relationship between whitening and \textit{spectral decoupling} \citep{pezeshki2021gradientstarvationlearningproclivity} in Appendix~\ref{app:SD}.

\subsection{Relation to standardization and uniform margins}
\label{sec:relations}
Whitening is closely related to two other approaches. The first is \textit{standardization}, 
where data is centered and scaled to have unit variance along each coordinate.
However, this does not create uniform eigenvalues, as there remain correlations between coordinates.
In addition, standardization is coordinate-dependent, while whitening is equivariant with respect to affine transformations of the data.

Second, whitening is also related to the concept of \textit{uniform margins} 
\citep{puli_2023}. Let
$m_i(\boldw):=\target_i\boldw^\top\boldx_i$ denote the signed training
margin. Since $\target_i^2=1$ and the data are centered,
\begin{align}
    \boldw^\top\hat{\boldSigma}\boldw
    =
    \frac{1}{n-1}\sum_{i=1}^n
    \bigl(\boldw^\top\boldx_i\bigr)^2
    =
    \frac{1}{n-1}\sum_{i=1}^n m_i(\boldw)^2.
\end{align}
Consequently, the whitened max-margin problem in
\eqref{eq:mm-whitened-data} is equivalent to
\begin{align}
    \min_{\boldw\in\calS}\;
    \frac{1}{n-1}\sum_{i=1}^n m_i(\boldw)^2
    \quad\text{s.t.}\quad
    m_i(\boldw)\geq1,\ \forall i.
\end{align}
The objective penalizes margins that are larger than required for
separation. When it is possible to have $m_i(\boldw)=1$ for all data points and to separate them, this solution is attained.
\citet{puli_2023} argue that uniform margins are desirable in the context of spurious correlations, providing an additional motivation for whitening.

\subsection{Estimating the covariance matrix}
\label{sec:est}

So far, we have used the empirical covariance matrix $\hat{\boldSigma}$ to whiten the data. 
In practice, this matrix can be nearly singular, especially when the number of features is large relative to the number of samples.
We then end up with very large values in $\hat{\boldLambda}^\dagger$.
This can lead to problems, as the scale of certain coordinates is massively inflated, despite the fact that they are not particularly informative.

A standard approach is to regularize the empirical covariance matrix by adding a small term to the eigenvalues \citep{Warton2008RidgeCovariance}.
We use the nonlinear shrinkage estimator of \citet{ledoit2020analytical}, which preserves the empirical eigenvectors but adjusts each eigenvalue separately, 
\begin{align}
    \hat{\boldSigma}_{\mathrm{LW}} =
\boldU_r
\operatorname{diag}
\left(
(\frac{\hat\lambda_1}{\hat\delta_1}), \dots,
(\frac{\hat\lambda_r}{\hat\delta_r}) \right)
\boldU_r^\top.
\end{align}
The adjustments $\hat\delta_i$ are chosen such that $\hat{\boldSigma}_{\mathrm{LW}}$ asymptotically minimizes the Frobenius norm of the difference with the covariance matrix $\boldSigma$ when $n,p\to\infty$ with $p/n$ converging to a positive constant.
We refer to $\hat{\boldSigma}_{\mathrm{LW}}$ as the \textit{nonlinear shrinkage estimator}, and unless otherwise mentioned use it when whitening. In Appendix~\ref{app:whitening-estimators} we compare whitening with this estimator to whitening with the empirical covariance matrix.

\section{Synthetic data-generating process}
\label{sec:synthetic-dgp}

We use a synthetic data-generating process (DGP) from the literature on spurious correlations \citep{pmlr-v119-sagawa20a,puli_2023, bayat2024pitfallsmemorizationmemorizationhurts} to illustrate the simplicity bias and the effect of whitening on the max-margin classifier.
We consider a training environment, labeled by $\mathrm{tr}$, and an out-of-distribution environment, labeled by $\mathrm{ood}$. Let $y \in\{-1,1\}$ denote the labels, and $a \in\{-1,1\}$ a spurious attribute.
The two environments have the same distribution of $\boldx$ given $(y,a)$; they differ only in the strength of the correlation between $y$ and $a$. We denote the environment via $e\in\{\mathrm{tr},\mathrm{ood}\}$. In environment $e$, let
\begin{align}
    a =
    \begin{cases}
        y & \text{with probability } \kappa_e,\\
        -y & \text{with probability } 1-\kappa_e,
    \end{cases}
    \qquad
    \rho_e:=2\kappa_e-1. \label{eq:synthetic-dgp}
\end{align}
Given $(y,a)$, we have the features $\boldx=(x_c,x_s,\boldsymbol{\epsilon}) \in \mathbb{R}^d$, where
\begin{align}
    x_s=\gamma x_a, \quad
    x_c \mid y \sim\mathcal{N}(y,\sigma_y^2), \quad
    x_a \mid a \sim\mathcal{N}(a,\sigma_a^2), \quad
    \boldsymbol{\epsilon}\sim
    \mathcal{N}\!\left(\boldsymbol{0},\frac{\sigma_\epsilon^2}{q}\bm{I}_q\right), \label{eq:synthetic-dgp-2}
\end{align}
where $q=d-2$. The first two coordinates, $x_c$ and $x_s$, are the `core' and `spurious'
features, respectively. The remaining $q$ features are (independent) noise.
The scaling of $\boldsymbol{\epsilon}$ gives
$\mathbb{E}\|\boldsymbol{\epsilon}\|_2^2=\sigma_\epsilon^2$. We center the dataset by subtracting the mean, and assume balanced labels, i.e., $\sum_{i=1}^{n_e}y_i=0$.

To empirically illustrate the effect of whitening for this synthetic DGP, we generate training data with $n=1000$ samples, $\gamma=5$, $\sigma_y=\sigma_a=0.1$, $\sigma_\epsilon=20$, $\kappa_{\mathrm{tr}}=0.9$.
We then fit a logistic regression on the original and whitened data, and evaluate the resulting models on a test dataset from the $\mathrm{ood}$ environment with $\kappa_{\mathrm{ood}}=0.5$.
Figure~\ref{fig:synthetic-decision-boundaries} shows how whitening changes the decision boundary induced by the max-margin classifier. 
With empirical risk minimization (ERM, \citet{Vapnik_1991}), which minimizes the logistic loss of the training data, the fitted coefficients rely to a greater extent on the spurious feature than when we use whitening.
\begin{figure}[H]
    \centering
    \includegraphics[width=\textwidth]{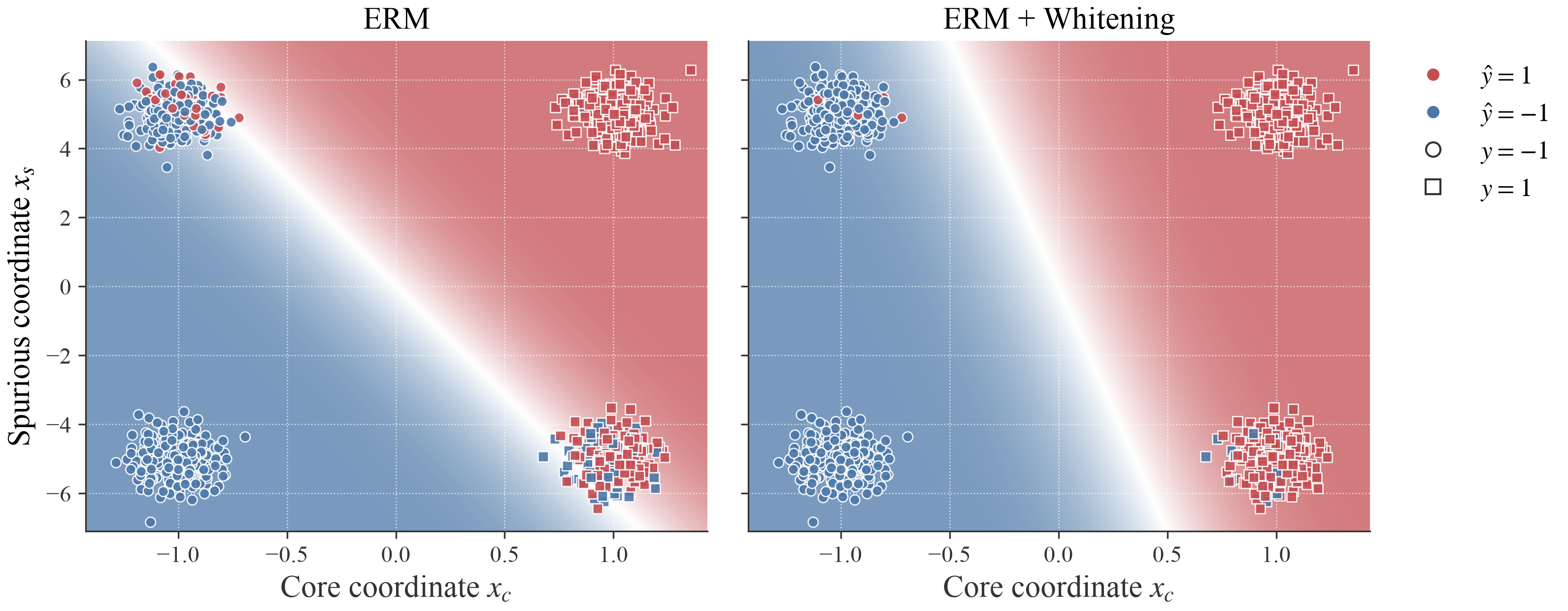}
    \caption{Predictions for ERM (left) and ERM after whitening (right) for $d=2000$, plotted against the core coordinate $x_c$ and spurious coordinate $x_s$, for 1000 sampled data points. Colors show predictions using all $d$ coordinates, and marker shapes show true classes. Background colors show the probability of predicting class $1$, averaged over the simulations, based on $x_c$ and $x_s$.}
    \label{fig:synthetic-decision-boundaries}
\end{figure} 
Previous work \citep{pmlr-v119-sagawa20a,puli_2023, bayat2024pitfallsmemorizationmemorizationhurts} has shown that the max-margin classifier in this DGP relies on the spurious feature $x_s$ when $\gamma$ is large enough.
This relates to the simplicity bias from Section~\ref{sec:theory}, as $\gamma$ inflates the scale (and the eigenvalues) of the spurious feature.
Data points that cannot be classified with the spurious feature are then fitted using the noise features. 

Figure~\ref{fig:synthetic-perf} illustrates how ERM fails to perform in the $\mathrm{ood}$ environment. As $q$ increases, there are enough noise coordinates that can be used to fit the minority groups ($y\neq a$, approximately 10\% of the sample) and classify the remainder with the spurious feature.
Whitening improves performance and yields nearly perfect overall and worst-group accuracy when $q<n$, and continues to strongly outperform standard ERM when $q\geq n$.

To formalize these empirical findings, we analyze the max-margin classifier for this DGP with whitening in a
specific asymptotic regime. In Theorem~\ref{thm:whitened-max-margin-limits}, we show the coefficients of the max-margin classifier fitted on whitened data converge to limits that mostly depend on the core feature.

\begin{figure}[h]
    \centering
    \includegraphics[width=\textwidth]{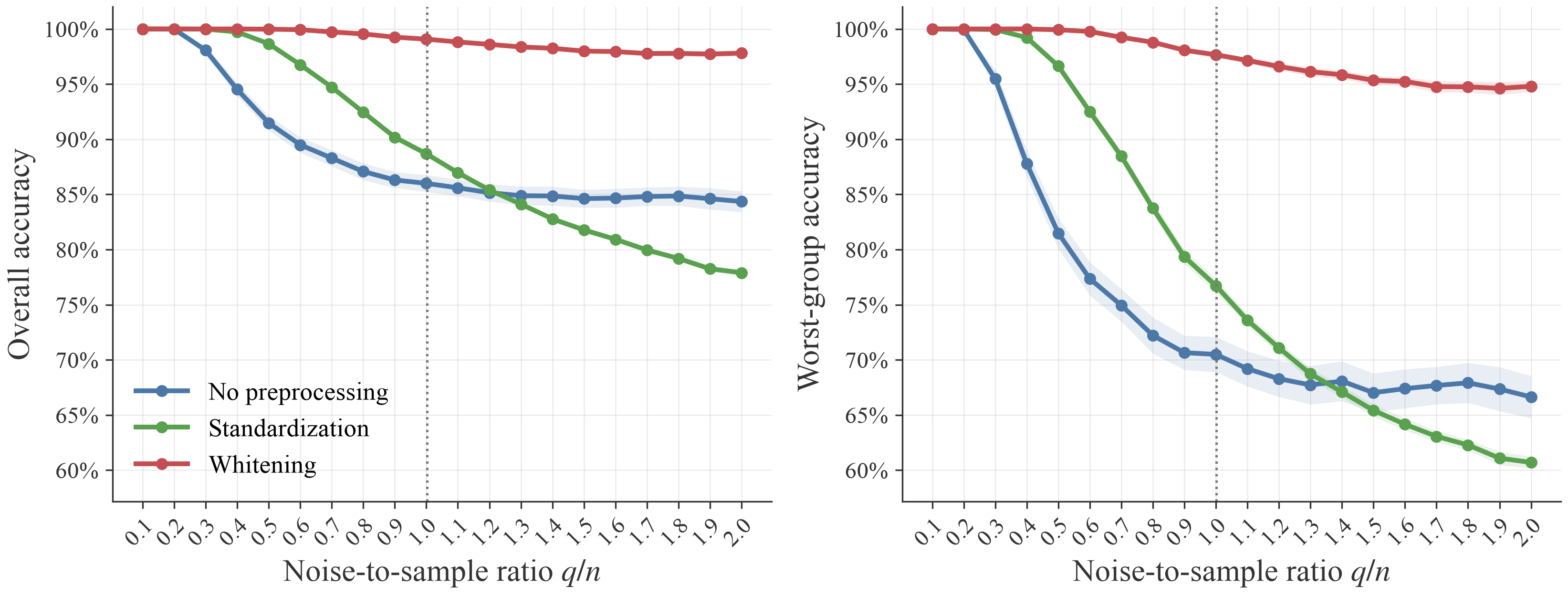}
    \caption{OOD test accuracy as a function of the noise-to-sample ratio $q/n$ for $n=1000$. We compare ERM with three preprocessing choices: no preprocessing, standardization, and whitening, in terms of overall accuracy (left) and worst-group accuracy (right), where groups are defined by the pair $(y, a)$. Points represent means over 100 simulations, shaded bands are 95\% confidence intervals.}
    \label{fig:synthetic-perf}
\end{figure}

\begin{theorem}[Whitened max-margin limits]
\label{thm:whitened-max-margin-limits}
Consider the synthetic DGP described in \eqref{eq:synthetic-dgp} and \eqref{eq:synthetic-dgp-2} with
$\gamma>0$, $\sigma_y^2=\sigma_a^2:=\sigma^2>0$,
$\sigma_\epsilon>0$, and $0<\rho_{\mathrm{tr}}<1$. Suppose that
the training labels are exactly balanced, $\sum_{i=1}^n y_i=0$. Let
$q,n\to\infty$ with $q/n\to\psi>1$. Define
\begin{align}
    \zeta:=1+\sigma^2,
    \qquad
    \tau^2
    :=
    \frac{\sigma^2(\zeta-\rho_{\mathrm{tr}}^2)}
         {\zeta^2-\rho_{\mathrm{tr}}^2}.
\end{align}
Then the max-margin classifier fitted on empirically whitened data, written
in the original coordinates as
$   \hat{\boldw}_{\mathrm{mm},\mathrm{whiten}}
    =
    \begin{pmatrix}
        \hat w_{\mathrm{whiten},c},
        \hat w_{\mathrm{whiten},s},
        \hat{\boldw}_{\mathrm{whiten},\epsilon}
    \end{pmatrix}^{\top},$
satisfies
\begin{align}
    \hat w_{\mathrm{whiten},c}
    \xrightarrow{p}
    \frac{\zeta-\rho_{\mathrm{tr}}^2}
    {\zeta^2-\rho_{\mathrm{tr}}^2},
    \quad 
    \hat w_{\mathrm{whiten},s}
    \xrightarrow{p}
    \frac{\rho_{\mathrm{tr}}\sigma^2}
    {\gamma(\zeta^2-\rho_{\mathrm{tr}}^2)}, \quad \frac1n
    \left\|
        \hat{\boldw}_{\mathrm{whiten},\epsilon}
    \right\|_2^2
    \xrightarrow{p}
    \frac{\psi}{\sigma_\epsilon^2(\psi-1)}\tau^2,
\end{align}
where $\xrightarrow{p}$ denotes convergence in
probability.
\end{theorem}
The proof is given in Appendix~\ref{app:white-proofs}. Whitening removes sensitivity to the
scales of the spurious features. If the scale of the spurious feature $\gamma$ increases, the coefficient becomes correspondingly smaller.
Likewise, the norm of the noise coefficients decreases as $\sigma_\epsilon$ increases.

One noteworthy effect of whitening is that when $\psi \rightarrow 1$, certain noise coordinates are severely inflated by whitening and subsequently used for classification.
This is the same issue as discussed in Section~\ref{sec:est} when not using the nonlinear shrinkage estimator.
    
In this proof, we have used that $\psi > 1$ ($q>n$). The central idea that
whitening leads to a classifier that primarily relies on the core feature also holds in an underparameterized
special case with $d=q+2<n$. In Appendix~\ref{app:whitening-uniform-margin},
we lay out how Theorem~2 of
\citet{puli_2023} can be used to show that whitening removes the simplicity bias in this case, under a special case of the synthetic DGP.

\section{Empirical evaluation}
\label{sec:experiments}

This section investigates whether whitening representations before fitting a linear
probe improves robustness to spurious correlations in practice. We consider three commonly used benchmarks:
Waterbirds \citep{sagawa2020distributionallyrobustneuralnetworks}, CelebA \citep{liu2015faceattributes},
and MultiNLI \citep{Bowman_MultiNLI}. Each benchmark contains a classification task in which the
target label is spuriously correlated with an additional attribute. For Waterbirds, the target is
bird type and the spurious attribute is the background; for CelebA, the target is hair color and the spurious attribute is gender; and for MultiNLI, the target is the relation between the premise and hypothesis and the spurious attribute is the presence of a negation word in the hypothesis.
The spurious correlation is present in both the training and validation sets. Details on each dataset can be found in Appendix~\ref{app:datasets}.

We use two widely used pretrained backbones for the benchmarks. For Waterbirds and CelebA,
we fine-tune an ImageNet-pretrained ResNet-50 \citep{He2016DeepResidual}, and a BERT model for MultiNLI. We then extract the final-layer representation and fit an $\ell_2$-regularized
logistic regression model as the linear probe. Results for more recent architectures are reported in Appendix~\ref{app:additional-empirical-results}.
Based on \citet{idrissi2022simpledatabalancingachieves} and \citet{labonte2023lastlayerretraininggrouprobustness}, 
we use class-balancing when fitting the logistic regression and selecting the $\ell_2$ penalty.
Details on the evaluation protocol can be found in Appendix~\ref{app:protocol}.

\textbf{Evaluation metrics:} We focus on worst-group accuracy. Groups are defined by the combination of the target label and the spurious attribute (e.g., bird type and background).
Worst-group accuracy is commonly used as a metric for robustness to spurious correlations \citep{sagawa2020distributionallyrobustneuralnetworks, pmlr-v119-sagawa20a},
as it captures performance on small groups where the attribute and target label conflict. We also report equal-group accuracy, an equally weighted average of accuracies per group. This indicates if improvements in worst-group accuracy come with a decrease in accuracy for other groups.

\textbf{Improving linear probes:} Figure~\ref{fig:empirical-main} reports worst-group and equal-group accuracy for a linear probe with three preprocessing variants:
no preprocessing, standardization, and whitening. For the two image benchmarks, we find whitening significantly improves both worst-group and equal-group accuracy over no preprocessing and standardization,
whereas for MultiNLI, the difference is statistically insignificant. The improvement from whitening can be substantial, for instance
providing a 13.5 percentage point (pp) increase in worst-group accuracy on Waterbirds compared to no preprocessing, and an even greater improvement (23.6 pp) when compared to standardization. 

\begin{figure}[h]
    \centering
    \includegraphics[width=\textwidth]{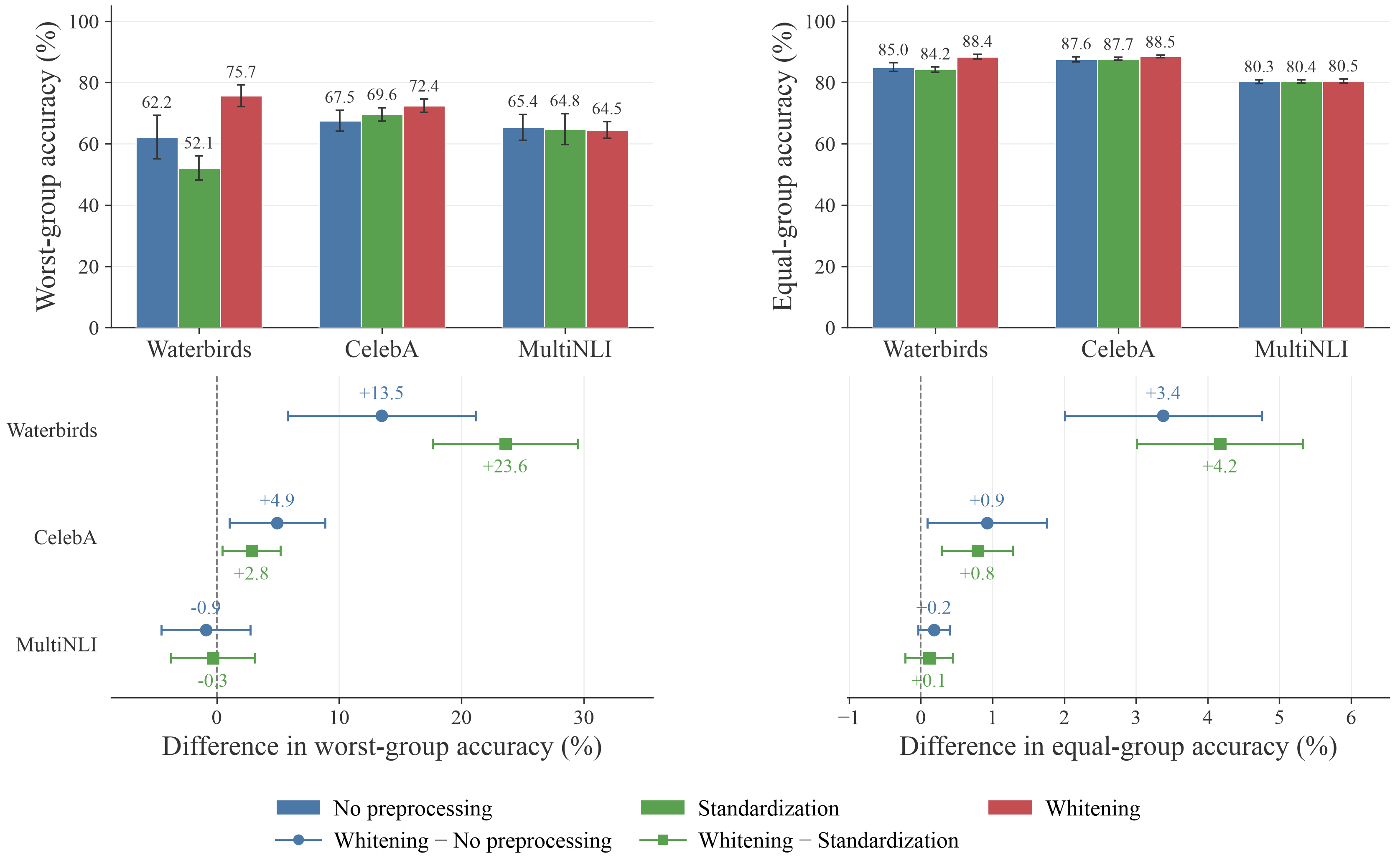}
    \caption{Worst-group (left) and equal-group (right) test accuracy for ERM with three preprocessing choices: no preprocessing, standardization, and whitening (top), together with mean within-seed paired differences for whitening minus each reference (bottom). Bars and points represent averages across seeds (ten for ResNet-50, five for BERT), and error bars are 95\% confidence intervals.}
    \label{fig:empirical-main}
\end{figure}

In order to investigate why whitening helps, we consider the eigenvalues of the empirical covariance matrix and the simplicity of the resulting classifiers (measured via the Rayleigh quotient, see Definition~\ref{def:simplicity}) of different classifiers in Figure~\ref{fig:empirical-spectral-diagnostics}.
For each representation, the empirical covariance matrix has a highly non-uniform set of eigenvalues. Figure~\ref{fig:empirical-spectral-diagnostics} also
shows that the classifier found after using ERM without preprocessing aligns with the largest eigenvalues, in line with the simplicity bias described in Section~\ref{sec:theory}.
The classifier found after whitening aligns with much smaller eigenvalues.
Thus, whitening appears to help because it reduces the tendency to rely on simple features.
However, it is not guaranteed that these simple features align with spurious features. 
This can explain the weaker result for MultiNLI: after whitening, the classifier relies on features that are less simple, with a similar worst-group accuracy.
Negation words are perhaps not `simple' and might be needed to classify entailment \citep{joshi-etal-2022-spurious}.
We also observe that standardization reduces simplicity but not to the same extent as whitening (likely for the reasons outlined in Section~\ref{sec:relations}), explaining why whitening outperforms standardization.

\begin{figure}[h]
    \centering
    \includegraphics[width=\textwidth]{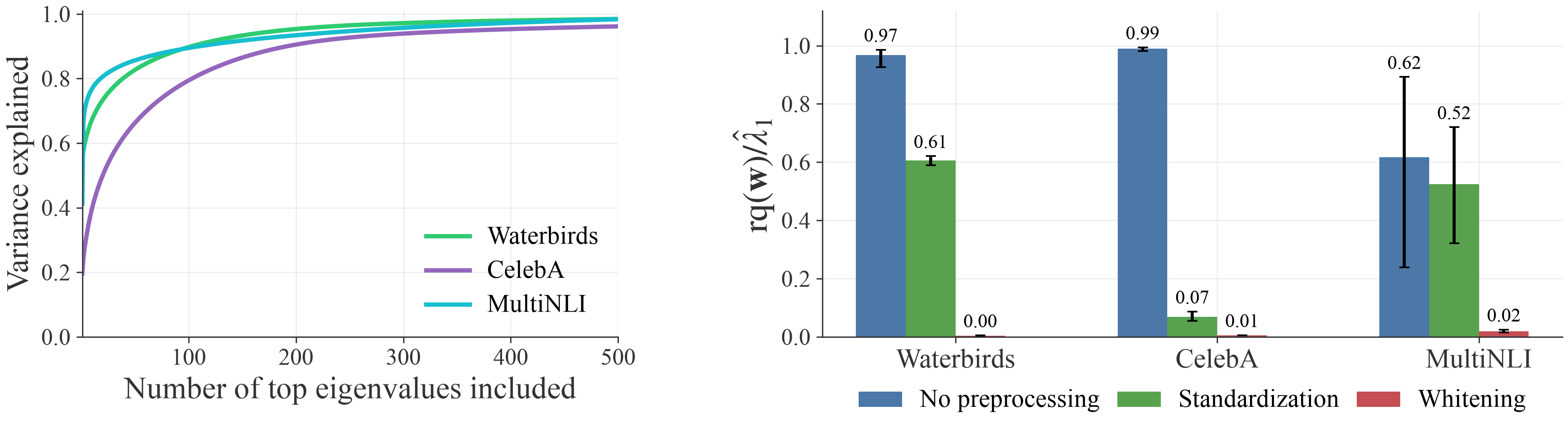}
    \caption{Variance explained by the top $k$ eigenvalues, up to $k=500$ (left), and the (normalized) Rayleigh quotient for different classifiers (right). 
    We take the classifier $\boldw$ after ERM with different preprocessing procedures and convert the classifier to the original coordinates, before calculating the Rayleigh quotient.
    Points and bars represent the average across seeds (ten for ResNet-50 and five for BERT), and error bars the 95\% confidence intervals.}
    \label{fig:empirical-spectral-diagnostics}
\end{figure}

\begin{figure}[h]
    \centering
    \includegraphics[width=\textwidth]{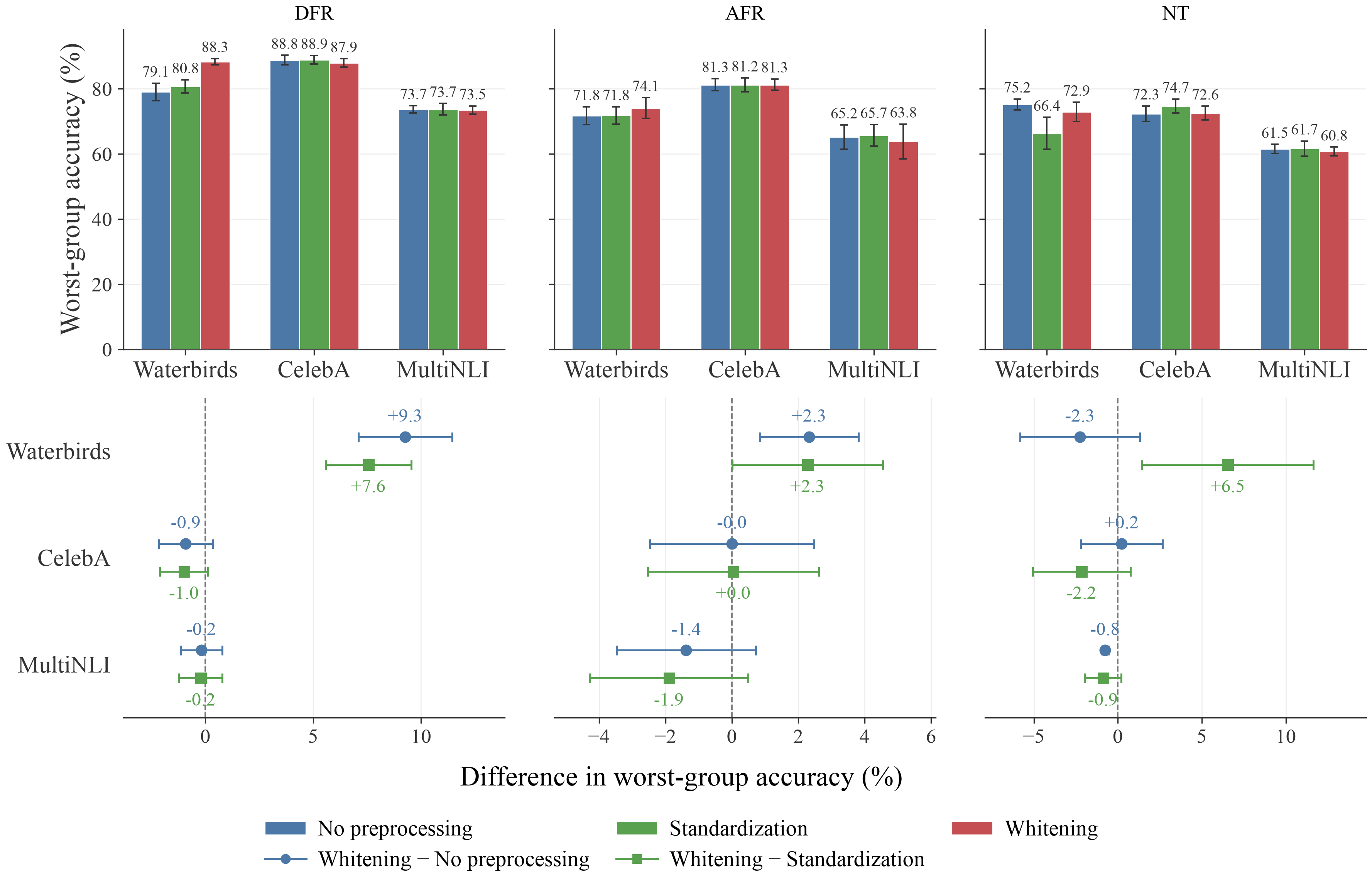}
    \caption{Worst-group test accuracy for DFR, AFR and NT for three preprocessing variants: no preprocessing, standardization, and whitening (top), together with mean within-seed paired differences for whitening minus each reference (bottom). Bars and points represent averages across seeds (ten for ResNet-50, five for BERT), and error bars are 95\% confidence intervals.}
    \label{fig:empirical-comparison}
\end{figure}

\textbf{Combining whitening and existing approaches:} we consider three state-of-the-art approaches that use linear probes in order to improve robustness to spurious correlations.
First, Deep Feature Reweighting (DFR) creates a group-balanced subsampled dataset based on the validation set, and fits a linear probe on it \citep{kirichenko2023layerretrainingsufficientrobustness}.
Second, Automatic Feature Reweighting (AFR) also fits a linear probe on a validation set, but reweights data points according to the predictions of a DNN trained with ERM \citep{qiu2023simplefastgrouprobustness}. 
Third, NeuronTune (NT) selects a subset of coordinates with high values when a model misclassifies a data point from the validation set \citep{pmlr-v267-zheng25s}, and removes these before fitting a linear probe.
Each approach uses different assumptions about the extent to which group labels are available: the full validation set for DFR, part of it for AFR, and none for NT.
For each, we assess the effect of using whitening before fitting the linear probe. Details on each approach are given in Appendix~\ref{app:other-approaches}.

Figure~\ref{fig:empirical-comparison} shows the worst-group accuracy for these three approaches without any preprocessing, standardization, and whitening.
Figure~\ref{fig:empirical-comparison-equal-group} in Appendix~\ref{app:additional-empirical-results} provides the equivalent comparison for equal-group accuracy.
Using whitening as a data-preprocessing step improves the worst-group accuracy for DFR and AFR for Waterbirds compared to doing no preprocessing or standardization, as well as for NeuronTune when compared to standardization.
For other datasets, the differences from using whitening are not statistically significant, with the exception of NT for MultiNLI, where whitening decreases the worst-group accuracy by a small amount (0.8 pp).

We attribute this to the fact that for other datasets, existing methods already reduce the simplicity of the resulting classifiers to a large extent. This can be seen in Figure~\ref{fig:simplicity-comparison}: 
for DFR, where we observe the largest improvement after whitening, we observe the largest decrease in simplicity of the resulting classifier.
In general, for the other datasets and approaches, the simplicity is already substantially lower before any whitening (compared to ERM without any preprocessing, as can be seen in Figure~\ref{fig:empirical-spectral-diagnostics}).

\begin{figure}[h]
    \centering
    \includegraphics[width=\textwidth]{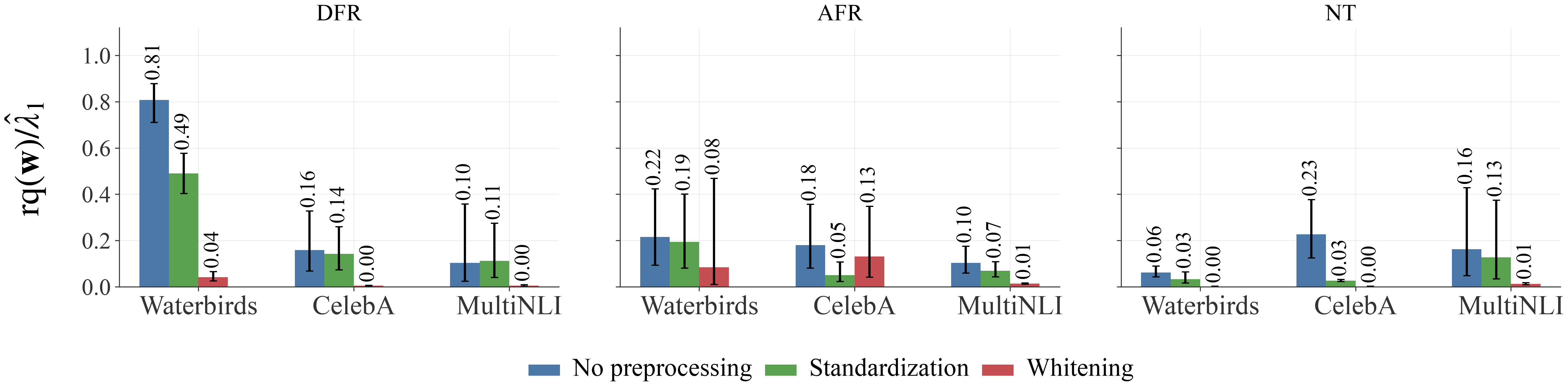}
    \caption{Normalized Rayleigh quotient for Deep Feature Reweighting (DFR), Automatic Feature Reweighting (AFR) and NeuronTune (NT) for three preprocessing variants: no preprocessing, standardization, and whitening. Points and bars represent the average across seeds (ten for ResNet-50 and five for BERT), and error bars are 95\% confidence intervals.}
    \label{fig:simplicity-comparison}
\end{figure}

\vspace{-0.5cm}
\section{Discussion and Conclusion}
\label{sec:conclusion}

In this paper, we show that the max-margin classifier exhibits a simplicity bias toward features associated with large eigenvalues of the empirical covariance matrix, and that whitening removes this bias.
On spurious correlation benchmarks, whitening can improve worst-group robustness, also when combined with existing methods that rely on linear probes.
Notably, in our results whitening only worsens the linear probe in a comparison. In addition, whitening requires neither pre-existing knowledge of the spurious correlation nor additional labels, and is computationally light.
As such, we argue that whitening as a preprocessing step should be strongly considered when addressing spurious correlations in linear probes.

A limitation of our approach is that it presumes simple features (with large eigenvalues) to be spurious.
While this has been argued in previous work (see Section~\ref{sec:related}), there are also cases where this might not hold, as we observe for MultiNLI.
In fact, it is sometimes desirable to encourage a classifier to make use of features with high eigenvalues \citep{dragutinović2026useusemuonsimplicity}. 
Whitening also does not address other reasons that linear probes rely on spurious features (see \citep{nagarajan2024understandingfailuremodesoutofdistribution}).

More broadly, our work suggests that modifying representation geometry is a simple way to reduce reliance on spurious features.
While we have shown the potential of whitening, we also find cases (especially when it is combined with existing approaches) where it does not lead to an improvement.
An interesting research direction to investigate is which geometry is `optimal' in terms of robustness against spurious correlations,
in the spirit of work characterizing desirable properties of representations (e.g. \citet{wang_isola_2020, balestriero2025lejepaprovablescalableselfsupervised}).

\clearpage

\subsection*{AI use statement}

In this work, we used generative AI tools to help with formulating mathematical claims, improving the writing of proofs, as well as suggesting strategies for proving a mathematical claim. In addition, we used generative AI tools to provide feedback on experiments as well as feedback on the implementation of other methods in the literature.
We have not used generative AI for developing the core argument laid out in the paper, to propose or refine hypotheses, or interpret results.
Finally, several other tasks that require disclosure - generating synthetic datasets, developing theoretical models or conceptual frameworks,  translation, dataset cleaning or reformatting, and qualitative or thematic data analysis - are not applicable to this work. 
Additionally, we used generative AI tools to create and edit code used for the experiments. We also used it to improve figures, the readability of the text, as well as finding spelling errors. We have reviewed all AI-assisted work.
Code written by generative AI was verified and tested for correctness by the first author. All four authors read the text of the paper, checking any suggestions provided by generative AI. Finally, any strategies suggested by generative AI for the proofs were verified by the first author.
We take responsibility for the final content of this work, including text, claims or artifacts produced with the aid of generative AI.

\subsection*{Ethics statement}

As with any technique that aims to improve the robustness of machine learning (ML) models to spurious correlations, practitioners should exercise caution when applying whitening in real-world settings where model predictions affect people’s lives. 
Reducing reliance on spurious correlations can improve the performance of an ML model for certain subgroups. 
However, our work focuses on a specific setting (a generalized linear model fitted on representations generated by DNNs) and evaluates the whitening on a limited number of benchmark datasets. 
These benchmarks do not capture all sources of bias and other considerations that may arise in deployment, such as unintended effects on other subgroups.
Moreover, our theoretical analysis relies on specific assumptions. The conclusions therefore need not extend directly to nonlinear models, or other distribution shifts.
Consequently, practitioners should evaluate whitening in the context in which it is to be deployed, and remain alert to potential unintended consequences.

\subsection*{Reproducibility statement}

The code used to generate the results in this paper will be made available after the review process in the form of a repository. We have attached it as a zip file to our submission.

\appendix

\newpage
\section{Proofs}
\label{app:proofs}

\subsection{Proof of Theorem~\ref{thm:whitened-max-margin-limits}}
\label{app:white-proofs}
\begin{proof}
The proof proceeds in two steps. We first show that the max-margin classifier on whitened data permits an analytical solution.
We then show the limits of this solution in the original coordinate system.
Throughout the proof,
the probability limits are taken as
\begin{align}
    n\to\infty,\qquad
    q\to\infty,\qquad
    q/n\to\psi>1.
\end{align}
Here, $\xrightarrow{p}$ denotes convergence in probability, under these limits. All
probabilities and expectations are conditional on $\sum_i y_i=0$.

As a preliminary step, we first introduce some notation and definitions.
We write the centered design matrix as
\begin{align}
    \boldsymbol X=(\boldsymbol C\ \boldsymbol E),
\end{align}
where $\boldsymbol C\in\mathbb R^{n\times 2}$ contains the sample-centered
core and spurious features and $\boldsymbol E\in\mathbb R^{n\times q}$ is
the centered noise block. Let $\boldsymbol y=(y_1,\ldots,y_n)^\top$ be the
corresponding labels. If
$\boldsymbol{1}:=(1,\ldots,1)^\top\in\mathbb R^n$ denotes the all-ones vector,
then centering places every column of $\boldsymbol X$ in
$\boldsymbol{1}^{\perp}:=\{\boldsymbol{u}\in\mathbb R^n:
\boldsymbol{1}^\top\boldsymbol{u}=0\}$.

Let
\[
\boldsymbol H
:=
\bm I_n-\frac1n\boldsymbol 1\boldsymbol 1^\top
\]
denote the centering matrix. Since the centered Gaussian noise block
$\boldsymbol E$ has $q$ columns supported on the $(n-1)$-dimensional space
$\boldsymbol 1^\perp$, and $q\ge n-1$ for all sufficiently large $n$,
\[
\operatorname{rank}(\boldsymbol E)=n-1
\]
almost surely. Consequently,
\[
\operatorname{col}(\boldsymbol E)
=
\operatorname{col}(\boldsymbol X)
=
\boldsymbol 1^\perp
\]
almost surely.

We start the first step by establishing the max-margin classifier on the
whitened data. Define
\begin{align}
    \widehat{\boldsymbol\Sigma}
    :=
    \frac1{n-1}\boldsymbol X^\top\boldsymbol X
    = \boldU_r\hat{\boldLambda}_r\boldU_r^\top,
\end{align}
where $r = n-1$ almost surely. Empirical whitening applies the map
$\boldU_r\hat{\boldLambda}_r^{-1/2}$ to the centered data, which gives
\begin{align}
    \boldsymbol Z
    :=
    \boldsymbol X\boldU_r\hat{\boldLambda}_r^{-1/2}.
\end{align}

When the data is whitened, the max-margin classifier is given by
\begin{align}
    \hat{\boldv}_{\mathrm{whiten}}
    =
    \arg\min_{\boldv}
    \|\boldv\|_2^2
    \quad
    \text{s.t.}
    \quad
    y_i\boldz_i^\top\boldv\ge 1,\quad i=1,\ldots,n, \label{eq:mm_full}
\end{align}
where $\boldz_i^\top$ is the $i$-th row of $\boldsymbol Z$. The
corresponding coefficient vector in the original coordinates is
\begin{align}
   \hat{\boldw}_{\mathrm{mm}, \mathrm{whiten}}
    =
    \boldU_r\hat{\boldLambda}_r^{-1/2}\hat{\boldv}_{\mathrm{whiten}}.
\end{align}
Consider the vector
\begin{align}
    \boldv_0
    :=
    \frac{n^{-1}\boldsymbol Z^\top\boldsymbol y}
    {\|n^{-1}\boldsymbol Z^\top\boldsymbol y\|_2^2}.
\label{eq:v0-definition}
\end{align}
We will show that $\boldv_0$ solves \eqref{eq:mm_full} by comparing the
problem with
\begin{align}
    \min_{\boldv}\|\boldv\|_2^2
    \quad\text{s.t.}\quad
    \left(
        \frac1n\boldsymbol Z^\top\boldsymbol y
    \right)^\top\boldv\geq1.
\label{eq:mm_limited}
\end{align}

Let
\begin{align}
    \mathcal F_{\mathrm{full}}
    &:={}
    \{\boldv:y_i\boldz_i^\top\boldv\geq1,\ i=1,\ldots,n\},
    \\
    \mathcal F_{\mathrm{avg}}
    &:={}
    \left\{
        \boldv:
        \left(
            \frac1n\boldsymbol Z^\top\boldsymbol y
        \right)^\top\boldv\geq1
    \right\}.
\end{align}
It is enough to verify the following two conditions:
\begin{enumerate}
    \item $\mathcal F_{\mathrm{full}}\subseteq\mathcal F_{\mathrm{avg}}$;
    \item $\boldv_0$ minimizes $\mathcal F_{\mathrm{avg}}$ and belongs to $\mathcal F_{\mathrm{full}}$.
\end{enumerate}
If these conditions hold, then
$\hat{\boldv}_{\mathrm{whiten}}=\boldv_0$.

We first verify condition 1. If $\boldv\in\mathcal F_{\mathrm{full}}$, then
$y_i\boldz_i^\top\boldv\geq1$ for every $i$. Since the $i$-th row of
$\boldsymbol Z$ is $\boldz_i^\top$, averaging these inequalities gives
\begin{align}
    \left(n^{-1}\boldsymbol Z^\top\boldsymbol y\right)^\top\boldv
    &= \frac1n\boldsymbol y^\top\boldsymbol Z\boldv \\
    &= \frac1n\sum_{i=1}^n y_i(\boldsymbol Z\boldv)_i \\
    &= \frac1n\sum_{i=1}^n y_i\boldz_i^\top\boldv \\
    &\geq \frac1n\sum_{i=1}^n 1 \\
    &= 1.
\end{align}
Thus $\boldv\in\mathcal F_{\mathrm{avg}}$, and hence
$\mathcal F_{\mathrm{full}}\subseteq\mathcal F_{\mathrm{avg}}$.

We next verify condition 2. For every $\boldv\in\mathcal F_{\mathrm{avg}}$,
the averaged-margin constraint and the Cauchy--Schwarz inequality imply
\begin{align}
    1&\leq \left(n^{-1}\boldsymbol Z^\top\boldsymbol y\right)^\top\boldv \\
    &\leq \left|\left(n^{-1}\boldsymbol Z^\top\boldsymbol y\right)^\top
       \boldv\right| \\
    &\leq \left\|n^{-1}\boldsymbol Z^\top\boldsymbol y\right\|_2
       \|\boldv\|_2.
\end{align}
By exact label balance, $\boldsymbol y\in\boldsymbol 1^\perp$. Since the
columns of $\boldsymbol Z$ span $\boldsymbol 1^\perp$ and
$\boldsymbol y\neq\boldsymbol 0$, we have
$\boldsymbol Z^\top\boldsymbol y\neq\boldsymbol 0$. Hence we may divide by
$\|n^{-1}\boldsymbol Z^\top\boldsymbol y\|_2$ to obtain
\begin{align}
    \|\boldv\|_2
    &\geq \frac{1}{\|n^{-1}\boldsymbol Z^\top\boldsymbol y\|_2} \\
    &= \frac{\|n^{-1}\boldsymbol Z^\top\boldsymbol y\|_2}
            {\|n^{-1}\boldsymbol Z^\top\boldsymbol y\|_2^2} \\
    &= \left\|\frac{n^{-1}\boldsymbol Z^\top\boldsymbol y}
            {\|n^{-1}\boldsymbol Z^\top\boldsymbol y\|_2^2}\right\|_2 \\
    &= \|\boldv_0\|_2.
\end{align}
Moreover, substituting $\boldv_0$ from \eqref{eq:v0-definition} into the
averaged-margin constraint gives
\begin{align}
    \left(n^{-1}\boldsymbol Z^\top\boldsymbol y\right)^\top\boldv_0
    &= \frac{\left(n^{-1}\boldsymbol Z^\top\boldsymbol y\right)^\top
             \left(n^{-1}\boldsymbol Z^\top\boldsymbol y\right)}
            {\|n^{-1}\boldsymbol Z^\top\boldsymbol y\|_2^2} = 1.
\end{align}
Thus $\boldv_0\in\mathcal F_{\mathrm{avg}}$ and attains the lower bound,
so it minimizes the squared norm over $\mathcal F_{\mathrm{avg}}$.

It remains to show that $\boldv_0$ is feasible for the full problem.
From the preliminary result above,
$\operatorname{col}(\boldsymbol X)=\boldsymbol 1^\perp$. Hence
$\boldsymbol X\boldsymbol X^\dagger$ is the orthogonal projector onto
$\boldsymbol 1^\perp$, namely $
\boldsymbol X\boldsymbol X^\dagger=\boldsymbol H
$. Using the definition of $\boldsymbol Z$ and
$\widehat{\boldsymbol\Sigma}=(n-1)^{-1}\boldsymbol X^\top\boldsymbol X$,
we obtain
\begin{align}
    \boldsymbol Z\boldsymbol Z^\top
    &= \boldsymbol X\boldU_r\hat{\boldLambda}_r^{-1}
       \boldU_r^\top\boldsymbol X^\top \\
    &= \boldsymbol X\widehat{\boldsymbol\Sigma}^\dagger
       \boldsymbol X^\top \\
    &= (n-1)\boldsymbol X
       (\boldsymbol X^\top\boldsymbol X)^\dagger\boldsymbol X^\top \\
    &= (n-1)\boldsymbol X\boldsymbol X^\dagger \\
    &= (n-1)\boldsymbol H.
\label{eq:white-gram}
\end{align}
Here we used the pseudoinverse identity
$(\boldsymbol X^\top\boldsymbol X)^\dagger\boldsymbol X^\top
=\boldsymbol X^\dagger$. Exact label balance gives
$\boldsymbol 1^\top\boldsymbol y=0$, and hence
\begin{align}
    \boldsymbol H\boldsymbol y
    &= \left(\bm I_n-\frac1n\boldsymbol 1\boldsymbol 1^\top\right)
       \boldsymbol y \\
    &= \boldsymbol y-\frac1n\boldsymbol 1
       (\boldsymbol 1^\top\boldsymbol y) \\
    &= \boldsymbol y.
\end{align}
Using \eqref{eq:white-gram}, $\boldsymbol H\boldsymbol y=\boldsymbol y$,
and $y_i^2=1$ for every $i$, we obtain
\begin{align}
    \left\|n^{-1}\boldsymbol Z^\top\boldsymbol y\right\|_2^2
    &= \frac1{n^2}\boldsymbol y^\top
       \boldsymbol Z\boldsymbol Z^\top\boldsymbol y \\
    &= \frac{n-1}{n^2}\boldsymbol y^\top\boldsymbol H\boldsymbol y \\
    &= \frac{n-1}{n^2}\boldsymbol y^\top\boldsymbol y \\
    &= \frac{n-1}{n^2}\sum_{i=1}^n y_i^2 \\
    &= \frac{n-1}{n}.
\end{align}
Substituting this value into \eqref{eq:v0-definition} gives
\begin{align}
    \boldv_0
    &= \frac{n^{-1}\boldsymbol Z^\top\boldsymbol y}{(n-1)/n} \\
    &= \frac{n}{n-1}\frac1n\boldsymbol Z^\top\boldsymbol y \\
    &= \frac1{n-1}\boldsymbol Z^\top\boldsymbol y.
\end{align}
Therefore,
\begin{align}
    \boldsymbol Z\boldv_0
    =
    \frac1{n-1}\boldsymbol Z\boldsymbol Z^\top\boldsymbol y
    =
    \frac1{n-1}(n-1)\boldsymbol H\boldsymbol y
    =
    \boldsymbol y.
\end{align}
Hence $y_i\boldz_i^\top\boldv_0=y_i^2=1$ for every $i$, so
$\boldv_0\in\mathcal F_{\mathrm{full}}$. The two conditions are therefore
satisfied, and consequently
\begin{align}
    \hat{\boldv}_{\mathrm{whiten}}
    =
    \boldv_0
    =
    \frac1{n-1}\boldsymbol Z^\top\boldsymbol y.
\end{align}
The corresponding original-coordinate coefficient vector is
\begin{align}
   \hat{\boldw}_{\mathrm{mm}, \mathrm{whiten}}
    =
    \boldU_r\hat{\boldLambda}_r^{-1/2}\hat{\boldv}_{\mathrm{whiten}}
    =
    \frac1{n-1}
    \widehat{\boldsymbol\Sigma}^\dagger
    \boldsymbol X^\top\boldsymbol y.
\label{eq:white-original-coef}
\end{align}

We now proceed with the second step of the proof, which is to take limits of \eqref{eq:white-original-coef}.
Let
\begin{align}
    \boldsymbol A:=\boldsymbol E\boldsymbol E^\top,
    \qquad
    \boldsymbol c:=
    \begin{pmatrix}x_c\\x_s\end{pmatrix}.
\end{align}
Here and below, $\mathbb E_{\mathrm{tr}}$ denotes expectation under the
training environment. Then
\begin{align}
    \boldsymbol\Sigma_{c,s}
    &:=
    \mathbb E_{\mathrm{tr}}[\boldsymbol c\boldsymbol c^\top]
    =
    \begin{pmatrix}
       1 + \sigma^2& \gamma\rho_{\mathrm{tr}}\\
        \gamma\rho_{\mathrm{tr}} & \gamma^2(1 + \sigma^2)
    \end{pmatrix},
    \qquad
    \boldsymbol\beta_{c,s}^{\mathrm{tr}}
    :=
    \mathbb E_{\mathrm{tr}}[y\boldsymbol c]
    =
    \begin{pmatrix}
        1\\
        \gamma\rho_{\mathrm{tr}}
    \end{pmatrix}.
\end{align}

Using \eqref{eq:white-original-coef}, we can write
\begin{align}
   \hat{\boldw}_{\mathrm{mm}, \mathrm{whiten}}
    =
    \boldsymbol X^\top
    \left(
        \boldsymbol X\boldsymbol X^\top
    \right)^\dagger
    \boldsymbol y .
\end{align}
Since
$\boldsymbol X\boldsymbol X^\top
=\boldsymbol A+\boldsymbol C\boldsymbol C^\top$,
the first two coordinates are given by
\begin{align}
    \begin{pmatrix}
        \hat w_{\mathrm{whiten},c}\\
        \hat w_{\mathrm{whiten},s}
    \end{pmatrix}
    =
    \boldsymbol C^\top
    \left(
        \boldsymbol A+
        \boldsymbol C\boldsymbol C^\top
    \right)^\dagger
    \boldsymbol y .
\label{eq:white-signal-direct}
\end{align}
Since $\operatorname{col}(\boldsymbol E)=\boldsymbol 1^\perp$,
$\boldsymbol A=\boldsymbol E\boldsymbol E^\top$ is invertible on
$\boldsymbol 1^\perp$. Moreover, $\boldsymbol C$ and $\boldsymbol y$ lie in
this same subspace. Applying the Woodbury identity on
$\boldsymbol 1^\perp$ therefore gives
\begin{align}
\left(
\boldsymbol A+
\boldsymbol C\boldsymbol C^\top
\right)^\dagger
&=
\boldsymbol A^\dagger
-
\boldsymbol A^\dagger
\boldsymbol C
\left(
\bm I_2+
\boldsymbol C^\top
\boldsymbol A^\dagger
\boldsymbol C
\right)^{-1}
\boldsymbol C^\top
\boldsymbol A^\dagger .
\label{eq:white-woodbury-gram}
\end{align}
Substituting \eqref{eq:white-woodbury-gram} into
\eqref{eq:white-signal-direct} and collecting terms gives
\begin{align}
    \begin{pmatrix}
        \hat w_{\mathrm{whiten},c}\\
        \hat w_{\mathrm{whiten},s}
    \end{pmatrix}
    &=
    \left[
        \bm I_2
        -
        \boldsymbol C^\top
        \boldsymbol A^\dagger
        \boldsymbol C
        \left(
            \bm I_2+
            \boldsymbol C^\top
            \boldsymbol A^\dagger
            \boldsymbol C
        \right)^{-1}
    \right]
    \boldsymbol C^\top
    \boldsymbol A^\dagger
    \boldsymbol y
    \\
    &=
    \left[
        \left(
            \bm I_2+
            \boldsymbol C^\top
            \boldsymbol A^\dagger
            \boldsymbol C
        \right)
        -
        \boldsymbol C^\top
        \boldsymbol A^\dagger
        \boldsymbol C
    \right]
    \left(
        \bm I_2+
        \boldsymbol C^\top
        \boldsymbol A^\dagger
        \boldsymbol C
    \right)^{-1}
    \boldsymbol C^\top
    \boldsymbol A^\dagger
    \boldsymbol y
    \\
    &=
    \left(
        \bm I_2+
        \boldsymbol C^\top
        \boldsymbol A^\dagger
        \boldsymbol C
    \right)^{-1}
    \boldsymbol C^\top
    \boldsymbol A^\dagger
    \boldsymbol y.
\label{eq:white-woodbury}
\end{align}
The centered noise block remains independent of
$(\boldsymbol C,\boldsymbol y)$. Its columns are Gaussian with covariance
$(\sigma_\epsilon^2/q)\boldsymbol H$. Define
\begin{align}
    \alpha_\psi := \frac{\psi}{\sigma_\epsilon^2(\psi-1)}.
\end{align}
Let $\boldu\in\boldsymbol 1^\perp$ be a nonzero vector.
Expressing $\boldsymbol A$ in an orthonormal basis of $\boldsymbol 1^\perp$
removes its zero direction $\boldsymbol 1$ and gives an invertible
$(n-1)\times(n-1)$ Wishart matrix with $q$ degrees of freedom and scale
$(\sigma_\epsilon^2/q)\bm I_{n-1}$.
For fixed $\boldu$, the standard
inverse-Wishart quadratic-form identity gives
\begin{align}
    \frac{q\|\boldu\|_2^2}
    {\sigma_\epsilon^2\boldu^\top\boldsymbol A^\dagger\boldu}
    \sim\chi^2_{q-n+2}.
\end{align}
If $Z_n\sim\chi^2_{q-n+2}$, its mean and variance imply
$Z_n/(q-n+2)\xrightarrow{p}1$. Hence
\begin{align}
    \frac{\boldu^\top\boldsymbol A^\dagger\boldu}{\|\boldu\|_2^2}
    &\overset{d}{=}
    \frac{q}{\sigma_\epsilon^2 Z_n}
    =\frac{1}{\sigma_\epsilon^2}
      \frac{q}{q-n+2}\frac{q-n+2}{Z_n}
    \xrightarrow{p}\alpha_\psi. \label{eq:limit-inverse-wishart}
\end{align}
Conditional on $(\boldsymbol C,\boldsymbol y)$, their entries are fixed,
while independence ensures that the noise block retains its original
Gaussian distribution. We can therefore apply the limit in \eqref{eq:limit-inverse-wishart} to each of the two-dimensional blocks of
$\boldsymbol C^\top\boldsymbol A^\dagger\boldsymbol C$ and
$\boldsymbol C^\top\boldsymbol A^\dagger\boldsymbol y$. This gives
\begin{align}
    \frac1n
    \boldsymbol C^\top
    \boldsymbol A^\dagger
    \boldsymbol C
    \xrightarrow{p}
    \alpha_\psi \boldsymbol\Sigma_{c,s},
    \qquad
    \frac1n
    \boldsymbol C^\top
    \boldsymbol A^\dagger
    \boldsymbol y
    \xrightarrow{p}
    \alpha_\psi \boldsymbol\beta_{c,s}^{\mathrm{tr}},
\label{eq:white-normalized-blocks-simple}
\end{align}
where we use that $\boldsymbol C^\top\boldsymbol C/n$ and
$\boldsymbol C^\top\boldsymbol y/n$ are sample averages converging to
$\boldsymbol\Sigma_{c,s}$ and $\boldsymbol\beta_{c,s}^{\mathrm{tr}}$,
respectively.
Rewriting \eqref{eq:white-woodbury} as
\begin{align}
    \begin{pmatrix}
        \hat w_{\mathrm{whiten},c}\\
        \hat w_{\mathrm{whiten},s}
    \end{pmatrix}
    =
    \left(
        \frac{\bm I_2}{n}
        +
        \frac1n
        \boldsymbol C^\top
        \boldsymbol A^\dagger
        \boldsymbol C
    \right)^{-1}
    \frac1n
    \boldsymbol C^\top
    \boldsymbol A^\dagger
    \boldsymbol y,
\end{align}
Slutsky's theorem and the continuous mapping theorem give
\begin{align}
    \begin{pmatrix}
        \hat w_{\mathrm{whiten},c}\\
        \hat w_{\mathrm{whiten},s}
    \end{pmatrix}
    \xrightarrow{p}
    \left(
        \alpha_\psi\boldsymbol\Sigma_{c,s}
    \right)^{-1}
    \left(
        \alpha_\psi\boldsymbol\beta_{c,s}^{\mathrm{tr}}
    \right)
    =
    \boldsymbol\Sigma_{c,s}^{-1}
    \boldsymbol\beta_{c,s}^{\mathrm{tr}}.
\label{eq:white-coef-limit}
\end{align}
Multiplying out, and writing $\zeta:=1+\sigma^2$, gives
\begin{align}
    \boldsymbol\Sigma_{c,s}^{-1}
    \boldsymbol\beta_{c,s}^{\mathrm{tr}}
    =
    \begin{pmatrix}
        \dfrac{\zeta-\rho_{\mathrm{tr}}^2}{\zeta^2-\rho_{\mathrm{tr}}^2}
        \\[1.2em]
        \dfrac{\rho_{\mathrm{tr}}\sigma^2}
        {\gamma(\zeta^2-\rho_{\mathrm{tr}}^2)}
    \end{pmatrix}.
\end{align}
This proves the claimed limits for the core and spurious coefficients.

It remains to track the norm of the noise coefficients. Let
\begin{align}
    \boldb_{c,s}
    &:=
    \boldsymbol\Sigma_{c,s}^{-1}
    \boldsymbol\beta_{c,s}^{\mathrm{tr}},
    \qquad
    \hat{\boldb}_{c,s}
    :=
    \begin{pmatrix}
        \hat w_{\mathrm{whiten},c}\\
        \hat w_{\mathrm{whiten},s}
    \end{pmatrix}.
\end{align}
Extracting the noise coordinates from \eqref{eq:white-original-coef} gives
\begin{align}
    \hat{\boldw}_{\mathrm{whiten},\epsilon}
    &=
    \boldsymbol E^\top
    \left(
        \boldsymbol A+
        \boldsymbol C\boldsymbol C^\top
    \right)^\dagger
    \boldsymbol y
    \\
    &=
    \boldsymbol E^\top
    \left[
        \boldsymbol A^\dagger
        -
        \boldsymbol A^\dagger
        \boldsymbol C
        \left(
            \bm I_2+
            \boldsymbol C^\top
            \boldsymbol A^\dagger
            \boldsymbol C
        \right)^{-1}
        \boldsymbol C^\top
        \boldsymbol A^\dagger
    \right]
    \boldsymbol y
    \\
    &=
    \boldsymbol E^\top
    \boldsymbol A^\dagger
    \left[
        \boldsymbol y
        -
        \boldsymbol C
        \left(
            \bm I_2+
            \boldsymbol C^\top
            \boldsymbol A^\dagger
            \boldsymbol C
        \right)^{-1}
        \boldsymbol C^\top
        \boldsymbol A^\dagger
        \boldsymbol y
    \right]
    \\
    &=
    \boldsymbol E^\top
    \boldsymbol A^\dagger
    \left(
        \boldsymbol y
        -
        \boldsymbol C\hat{\boldb}_{c,s}
    \right).
\label{eq:white-noise-woodbury}
\end{align}
The second equality applies \eqref{eq:white-woodbury-gram}, the third factors
out $\boldsymbol E^\top\boldsymbol A^\dagger$, and the final equality applies
\eqref{eq:white-woodbury}.
Define the residual vector
\begin{align}
    \boldsymbol r_n
    :=
    \boldsymbol y
    -
    \boldsymbol C\hat{\boldb}_{c,s}.
\end{align}
Since $\boldsymbol A=\boldsymbol E\boldsymbol E^\top$,
we have
\begin{align}
    \left\|
        \hat{\boldw}_{\mathrm{whiten},\epsilon}
    \right\|_2^2
    =
    \boldsymbol r_n^\top
    \boldsymbol A^\dagger
    \boldsymbol E
    \boldsymbol E^\top
    \boldsymbol A^\dagger
    \boldsymbol r_n
    =
    \boldsymbol r_n^\top
    \boldsymbol A^\dagger
    \boldsymbol r_n .
\end{align}
The last equality uses $\boldsymbol A=\boldsymbol E\boldsymbol E^\top$ and
the Moore--Penrose identity
$\boldsymbol A^\dagger\boldsymbol A\boldsymbol A^\dagger=\boldsymbol A^\dagger$.

To take the limit of this norm, we expand its quadratic form directly:
\begin{align}
    \frac1n
    \left\|
        \hat{\boldw}_{\mathrm{whiten},\epsilon}
    \right\|_2^2
    &=
    \frac1n
    \boldsymbol y^\top
    \boldsymbol A^\dagger
    \boldsymbol y
    -
    2\hat{\boldb}_{c,s}^\top
    \left(
        \frac1n
        \boldsymbol C^\top
        \boldsymbol A^\dagger
        \boldsymbol y
    \right)
    +
    \hat{\boldb}_{c,s}^\top
    \left(
        \frac1n
        \boldsymbol C^\top
        \boldsymbol A^\dagger
        \boldsymbol C
    \right)
    \hat{\boldb}_{c,s}.
\end{align}
Since $y^2_i = 1$, we have $\boldsymbol y^\top\boldsymbol y/n=1$. The same
inverse-Wishart quadratic-form limit used in
\eqref{eq:white-normalized-blocks-simple} therefore gives
\begin{align}
    \frac1n
    \boldsymbol y^\top
    \boldsymbol A^\dagger
    \boldsymbol y
    \xrightarrow{p}
    \alpha_\psi.
\end{align}
Combining this limit with \eqref{eq:white-normalized-blocks-simple} and
$\hat{\boldb}_{c,s}\xrightarrow{p}\boldb_{c,s}$ gives
\begin{align}
    \frac1n
    \left\|
        \hat{\boldw}_{\mathrm{whiten},\epsilon}
    \right\|_2^2
    \xrightarrow{p}
    \alpha_\psi
    \left(
        1
        -
        2\boldb_{c,s}^\top\boldsymbol\beta_{c,s}^{\mathrm{tr}}
        +
        \boldb_{c,s}^\top
        \boldsymbol\Sigma_{c,s}
        \boldb_{c,s}
    \right)
    =
    \alpha_\psi
    \left[
        1-
        \left(
            \boldsymbol\beta_{c,s}^{\mathrm{tr}}
        \right)^\top
        \boldsymbol\Sigma_{c,s}^{-1}
        \boldsymbol\beta_{c,s}^{\mathrm{tr}}
    \right].
\end{align}
The remaining term simplifies to
\begin{align}
    \left(
        \boldsymbol\beta_{c,s}^{\mathrm{tr}}
    \right)^\top
    \boldsymbol\Sigma_{c,s}^{-1}
    \boldsymbol\beta_{c,s}^{\mathrm{tr}}
    =
    \begin{pmatrix}
        1 & \gamma\rho_{\mathrm{tr}}
    \end{pmatrix}
    \begin{pmatrix}
        \dfrac{\zeta-\rho_{\mathrm{tr}}^2}
        {\zeta^2-\rho_{\mathrm{tr}}^2}
        \\[1em]
        \dfrac{\rho_{\mathrm{tr}}\sigma^2}
        {\gamma(\zeta^2-\rho_{\mathrm{tr}}^2)}
    \end{pmatrix}
    =
    \frac{
        \zeta-\rho_{\mathrm{tr}}^2
        +\rho_{\mathrm{tr}}^2\sigma^2
    }{
        \zeta^2-\rho_{\mathrm{tr}}^2
    }.
\end{align}
Hence, using $\sigma^2=\zeta-1$,
\begin{align}
    1-
    \left(
        \boldsymbol\beta_{c,s}^{\mathrm{tr}}
    \right)^\top
    \boldsymbol\Sigma_{c,s}^{-1}
    \boldsymbol\beta_{c,s}^{\mathrm{tr}}
    &=
    1-
    \frac{
        \zeta-\rho_{\mathrm{tr}}^2
        +\rho_{\mathrm{tr}}^2\sigma^2
    }{
        \zeta^2-\rho_{\mathrm{tr}}^2
    }
    \\
    &=
    \frac{
        \zeta^2-\zeta-\rho_{\mathrm{tr}}^2\sigma^2
    }{
        \zeta^2-\rho_{\mathrm{tr}}^2
    }
    \\
    &=
    \frac{
        (\zeta-1)(\zeta-\rho_{\mathrm{tr}}^2)
    }{
        \zeta^2-\rho_{\mathrm{tr}}^2
    }
    \\
    &=
    \frac{
        \sigma^2(\zeta-\rho_{\mathrm{tr}}^2)
    }{
        \zeta^2-\rho_{\mathrm{tr}}^2
    }
    =
    \tau^2.
\end{align}
Therefore,
\begin{align}
    \frac1n
    \left\|
        \hat{\boldw}_{\mathrm{whiten},\epsilon}
    \right\|_2^2
    \xrightarrow{p}
    \alpha_\psi \tau^2
    =
    \frac{\psi}{\sigma_\epsilon^2(\psi-1)}
    \tau^2.
\end{align}
This gives the stated limit for the norm of the noise coefficients and
completes the proof.
\end{proof}

\subsection{Whitening selects coefficients with uniform margins}
\label{app:whitening-uniform-margin}

We consider the uncentered version of the DGP in
Section~\ref{sec:synthetic-dgp}. In particular, let
\begin{align}
    \boldx_i
    =
    (y_i,\gamma a_i,\boldsymbol\epsilon_i),
    \qquad
    \boldsymbol\epsilon_i\sim\mathcal N(\boldsymbol 0,\bm I_q),
\end{align}
which corresponds to $\sigma_y^2=\sigma_a^2=0$ and
$\sigma_\epsilon^2=q$ in our parameterization.
We presume $0 < \rho_{\mathrm{tr}} < 1$.
We first state Theorem~2 of \citet{puli_2023} in our notation.

\begin{theorem}[Uniform margins; \citealp{puli_2023}, Theorem~2]
\label{thm:puli-uniform-margin}
Consider $n$ observations from the DGP above, with $\gamma>1$ and
$d=q+2<n$. Suppose that a linear model
$\boldw=(w_c,w_s,\boldw_\epsilon)$ has uniform positive margins: for some
$b>0$,
\begin{align}
    y_i\boldw^\top\boldx_i=b,
    \qquad i=1,\ldots,n.
\end{align}
Then, with probability one over the training sample,
\begin{align}
    \boldw=(b,0,\boldsymbol 0).
\end{align}
\end{theorem}

Thus, any set of coefficients with uniform margins in this DGP relies exclusively on the
core feature. It remains to show that the max-margin classifier fitted on
whitened data has uniform margins.

Because we are working with uncentered data, we work with the empirical second-moment matrix
\begin{align}
    \widetilde{\boldSigma}
    :=
    \frac{1}{n-1}\sum_{i=1}^n\boldx_i\boldx_i^\top.
\end{align}
Under the conditions of Theorem~\ref{thm:puli-uniform-margin},
$\widetilde{\boldSigma}$ is positive definite with probability one. For
$\widetilde{\boldz}_i:=\widetilde{\boldSigma}^{-1/2}\boldx_i$, set
$\boldw:=\widetilde{\boldSigma}^{-1/2}\boldv$. Then
$\boldv^\top\widetilde{\boldz}_i=\boldw^\top\boldx_i$ and
$\|\boldv\|_2^2=\boldw^\top\widetilde{\boldSigma}\boldw$. Consequently, the
max-margin classifier fitted on the whitened data can be written as
\begin{align}
    \hat{\boldw}_{\mathrm{mm},\mathrm{whiten}}
    =
    \arg\min_{\boldw\in\mathbb R^d}
    \boldw^\top\widetilde{\boldSigma}\boldw
    \quad\text{s.t.}\quad
    y_i\boldw^\top\boldx_i\geq1,
    \quad i=1,\ldots,n.
\label{eq:uncentered-whitened-mm}
\end{align}

\begin{proposition}[Core-only prediction after whitening]
\label{prop:uncentered-whitening-core-only}
Under the conditions of Theorem~\ref{thm:puli-uniform-margin}, the coefficients
in \eqref{eq:uncentered-whitened-mm} satisfy
\begin{align}
    \hat{\boldw}_{\mathrm{mm},\mathrm{whiten}}
    =
    (1,0,\boldsymbol 0)
\end{align}
with probability one.
\end{proposition}

\begin{proof}
For $m_i(\boldw):=y_i\boldw^\top\boldx_i$, we have
\begin{align}
    \boldw^\top\widetilde{\boldSigma}\boldw
    =
    \frac{1}{n-1}\sum_{i=1}^n
    \bigl(\boldw^\top\boldx_i\bigr)^2
    =
    \frac{1}{n-1}\sum_{i=1}^n m_i(\boldw)^2,
\end{align}
where the second equality uses $y_i^2=1$. Every set of coefficients in
\eqref{eq:uncentered-whitened-mm} therefore satisfies
\begin{align}
    \boldw^\top\widetilde{\boldSigma}\boldw
    \geq
    \frac{n}{n-1}.
\end{align}
The core-only coefficients $\boldw_c=(1,0,\boldsymbol 0)$ attain this lower
bound, since
\begin{align}
    m_i(\boldw_c)=y_i^2=1,
    \qquad i=1,\ldots,n.
\end{align}
Consequently, the minimizer of \eqref{eq:uncentered-whitened-mm} has uniform
margins equal to one. Theorem~\ref{thm:puli-uniform-margin}, applied with
$b=1$, then gives
\begin{align}
    \hat{\boldw}_{\mathrm{mm},\mathrm{whiten}}
    =
    (1,0,\boldsymbol 0),
\end{align}
which proves the result.
\end{proof}

\clearpage
\subsection{Numerical confirmation}

In this section, we numerically confirm Theorem~\ref{thm:whitened-max-margin-limits} and Proposition~\ref{prop:uncentered-whitening-core-only}.
In Figure~\ref{fig:theorem2-confirmation} we show the three terms that appear in Theorem~\ref{thm:whitened-max-margin-limits}: the
core coefficient, the spurious coefficient, and the normalized squared norm of the noise coefficients. For each, we analyze how they evolve as $q$ and $n$ grow, for $\psi = 1.2$.
All three quantities move toward their corresponding theoretical limits over
the displayed sample sizes.

\begin{figure}[htbp]
    \centering
    \includegraphics[width=\textwidth]{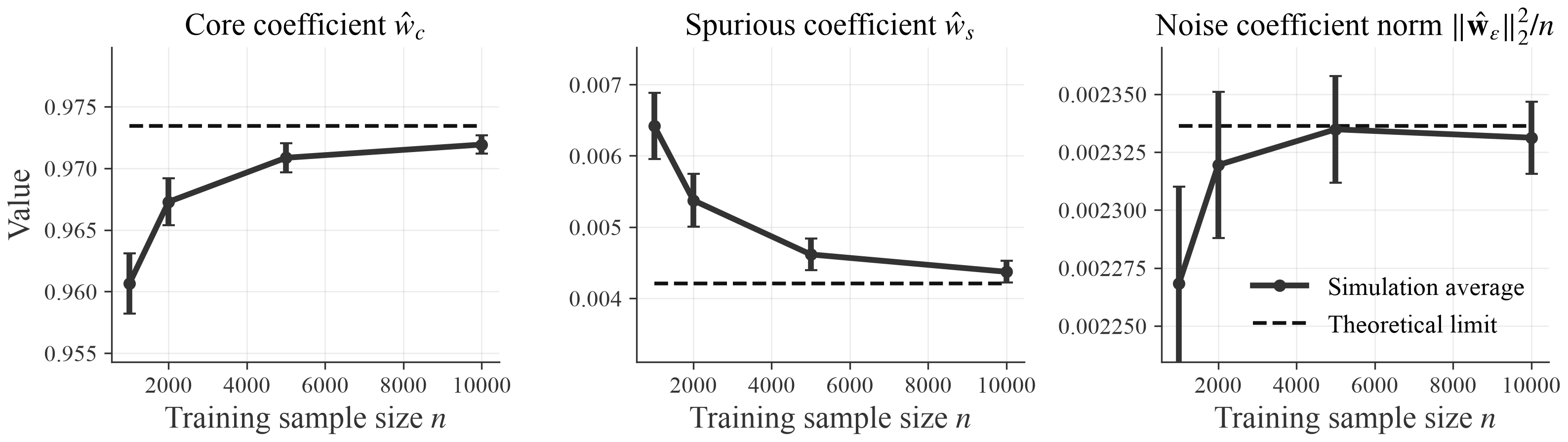}
    \caption{Estimated core coefficient $\hat w_c$ (left), spurious coefficient $\hat w_s$ (center), and normalized squared noise-coefficient norm $\lVert\hat{\boldw}_\epsilon\rVert_2^2/n$ (right) as functions of the training sample size $n$ for $q/n=1.2$. Solid black curves are simulation averages, vertical error bars are 95\% confidence intervals, and dashed lines are the corresponding limits in Theorem~\ref{thm:whitened-max-margin-limits}.}
    \label{fig:theorem2-confirmation}
\end{figure}

In Figure~\ref{fig:proposition2-confirmation} we show the same three fitted quantities for $\psi = 0.8$. It confirms Proposition~\ref{prop:uncentered-whitening-core-only}:
For every simulated training sample, we empirically whiten the uncentered
features using the second moment matrix. Because
logistic regression and the max-margin classifier only converge in direction, 
we divide each coefficient by its minimum signed training
margin. 

\begin{figure}[htbp]
    \centering
    \includegraphics[width=\textwidth]{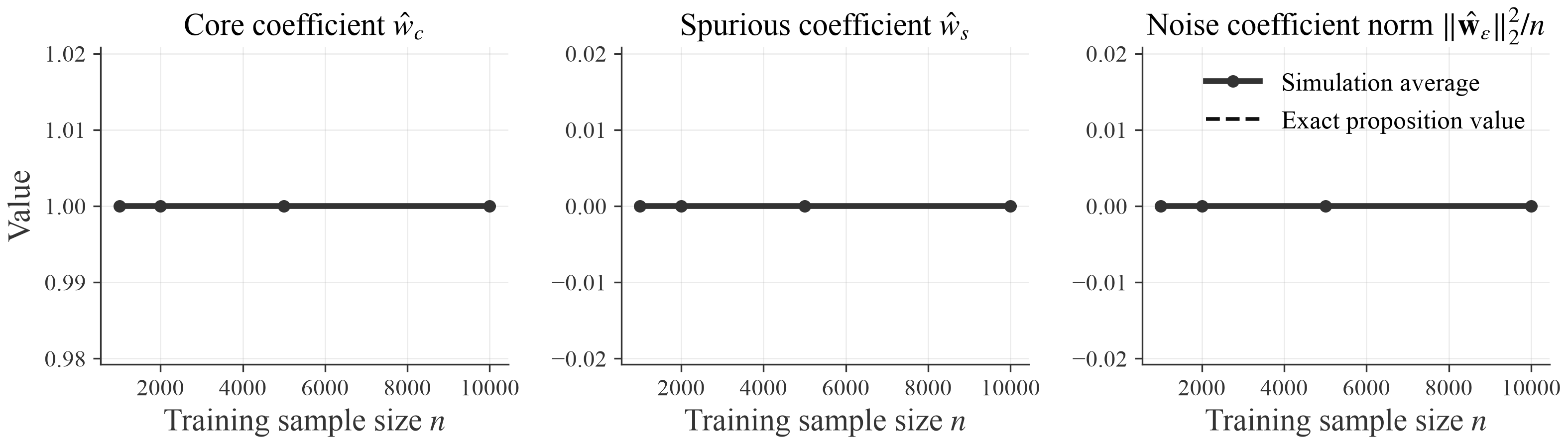}
    \caption{Fitted core coefficient $\hat w_c$ (left), spurious coefficient $\hat w_s$ (center), and normalized squared noise-coefficient norm $\lVert\hat{\boldw}_\epsilon\rVert_2^2/n$ (right) as functions of the training sample size $n$ for $q/n=0.8$. For each simulated sample, logistic regression is fitted after uncentered empirical whitening and the original-coordinate coefficients are normalized to unit minimum training margin. Solid black curves are simulation averages, vertical error bars are 95\% confidence intervals, and dashed lines are the exact core-only values from Proposition~\ref{prop:uncentered-whitening-core-only}.}
    \label{fig:proposition2-confirmation}
\end{figure}

\clearpage
\section{The relationship between whitening and spectral decoupling}
\label{app:SD}

A related approach to whitening is \textit{spectral decoupling} (SD) \citep{pezeshki2021gradientstarvationlearningproclivity}, which adds a penalty term to the loss function that is proportional to the squared norm of the model's predictions. For
a linear model, its objective can be written as
\begin{align}
\min_{\boldw \in \calS}
\frac1n\sum_{i=1}^n \log\bigl(1+\exp(-\target_i\boldw^\top \boldx_i)\bigr)
+\frac{\rho}{2n}\sum_{i=1}^n
\left(\boldw^\top\boldx_i-\gamma_{\target_i}\right)^2.
\label{eq:sd-obj}
\end{align}
For centered data and $\boldw \in \calS$, the change of variables in \eqref{eq:whiten-map} gives
\begin{align}
\target_i\,\boldw^\top\boldx_i = \target_i\,\boldv^\top\boldz_i
\qquad\text{and}\qquad
\boldw^\top\hat{\boldSigma}\,\boldw = \|\boldv\|_2^2.
\label{eq:sd-identities}
\end{align}
Substituting \eqref{eq:sd-identities} into \eqref{eq:sd-obj} and setting $\gamma_{\target_i}=0$ yields the optimization problem
\begin{align}
\min_{\boldv\in\mathcal{S}_z}\;
\frac1n\sum_{i=1}^n \log\!\bigl(1+\exp(-\target_i\,\boldv^\top \boldz_i)\bigr)
\;+\;\frac{\rho(n-1)}{2n}\,\|\boldv\|_2^2,
\label{eq:sd-equiv-whitened}
\end{align}
which is logistic regression on the whitened data
$\{\boldz_i\}_{i=1}^n$ with $\ell_2$ regularization.
The constraint $\boldv\in\mathcal{S}_z$ is redundant, since the logistic loss is unaffected by components of $\boldv$ in the null space of the whitened data.
Thus, SD with centered data and $\gamma_{\target_i}=0$ is equivalent to logistic regression on whitened data and an $\ell_2$ penalty.
The $\gamma_{\target_i}$ encourages logits $\boldw^\top\boldx_i$ to be close to a target value,
whereas whitening encourages the logits to be close to zero (while satisfying the margin constraints).
SD was motivated to address the simplicity bias of DNNs, and our analysis shows that whitening addresses the same issue for the max-margin classifier without this additional hyperparameter.

\section{Empirical whitening and the nonlinear shrinkage estimator}
\label{app:whitening-estimators}
\addtocounter{figure}{1}

In this section, we briefly investigate the difference between empirical whitening and whitening with the nonlinear shrinkage estimator introduced in Section~\ref{sec:est}.

We first focus on the synthetic data-generating process introduced in Section~\ref{sec:synthetic-dgp}.
Figure~\ref{fig:synthetic-perf-empirical} reports the empirical-whitening
counterpart to Figure~\ref{fig:synthetic-perf}. We observe poor performance of empirical whitening at $q/n =1$.
Theorem~\ref{thm:whitened-max-margin-limits} shows that as $n$ approaches $q$, whitening drastically scales the coefficients assigned to noise features.
Intuitively, when $q \approx n$, the empirical covariance matrix becomes nearly singular, so whitening strongly amplifies low-variance noise directions and the classifier can now rely on these to classify points.
This does not occur at $q < n$, as it becomes very unlikely the classifier can fit the training labels through noise.
With $q > n$, the classifier can spread the fit across more noise directions with smaller coefficients.

\begin{figure}[H]
    \centering
    \includegraphics[width=\textwidth]{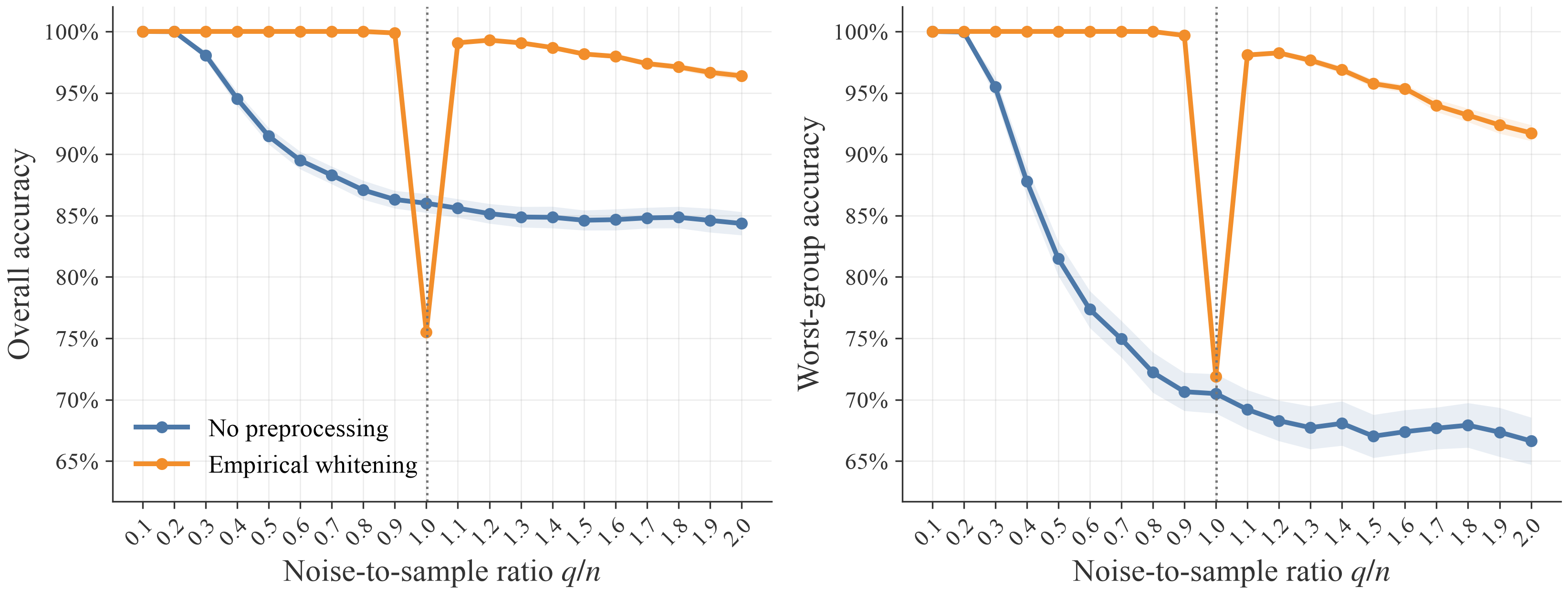}
    \caption{OOD test accuracy as a function of the noise-to-sample ratio $q/n$ for $n=1000$. We compare ERM with ERM after empirical whitening in terms of overall accuracy (left) and worst-group accuracy (right), where groups are defined by the pair $(y, a)$. Points are means over 100 simulations, shaded bands are 95\% confidence intervals, and the dotted line marks $q/n=1$.}
    \label{fig:synthetic-perf-empirical}
\end{figure}

\clearpage
Figure~\ref{fig:whitening-estimator-comparison} compares empirical whitening
with whitening with the nonlinear shrinkage estimator on the main empirical benchmarks. Although using the nonlinear shrinkage estimator improves the performance on average for Waterbirds and CelebA, the difference is not statistically significant.
For MultiNLI, the performance is close to equivalent, as there are many more data points relative to the dimension, and the difference between $\hat{\boldSigma}^{-1}$ and $\hat{\boldSigma}^{-1}_{\mathrm{LW}}$ becomes very small.

\begin{figure}[H]
    \centering
    \includegraphics[width=\textwidth]{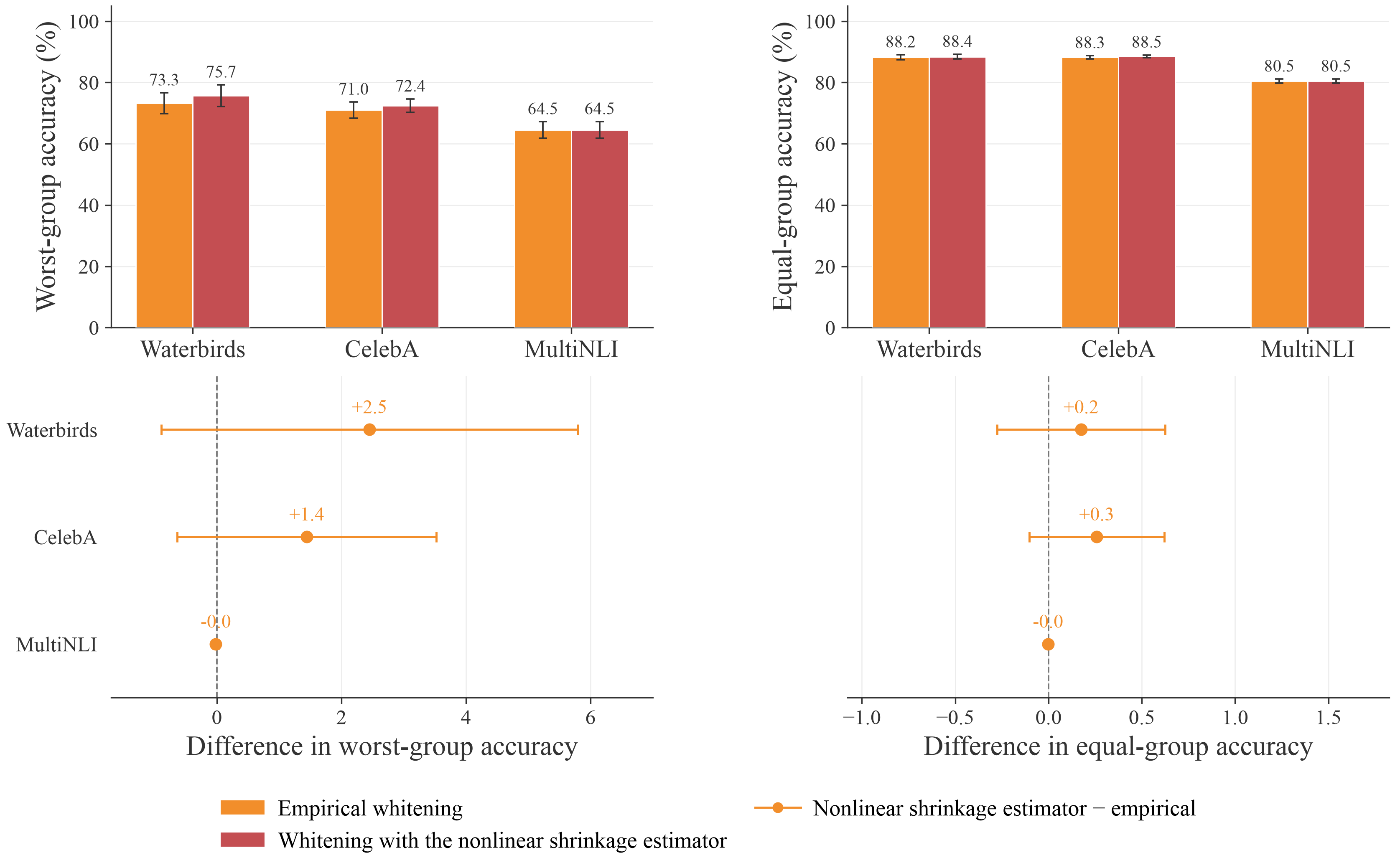}
    \caption{Worst-group (left) and equal-group (right) test accuracy for ERM with empirical whitening and whitening with the nonlinear shrinkage estimator (top), together with mean within-seed paired differences for whitening with the nonlinear shrinkage estimator minus empirical whitening (bottom). Bars and points represent averages across seeds (ten for ResNet-50, five for BERT), and error bars are 95\% confidence intervals.}
    \label{fig:whitening-estimator-comparison}
\end{figure}

\section{Additional empirical results}
\label{app:additional-empirical-results}

Figure~\ref{fig:empirical-appendix-backbones} extends the main comparison to
DINO ViT-B/16 on Waterbirds and CelebA and DeBERTa-v3-base on MultiNLI. We observe results similar to those in Figure~\ref{fig:empirical-main}, with the exception of the comparison between whitening and standardization for CelebA, where standardization and whitening now perform comparably.
We observe a much higher worst-group accuracy for DeBERTa-v3-base on MultiNLI than BERT, most likely due to the various improvements made since the BERT model.

\begin{figure}[h]
    \centering
    \includegraphics[width=\textwidth]{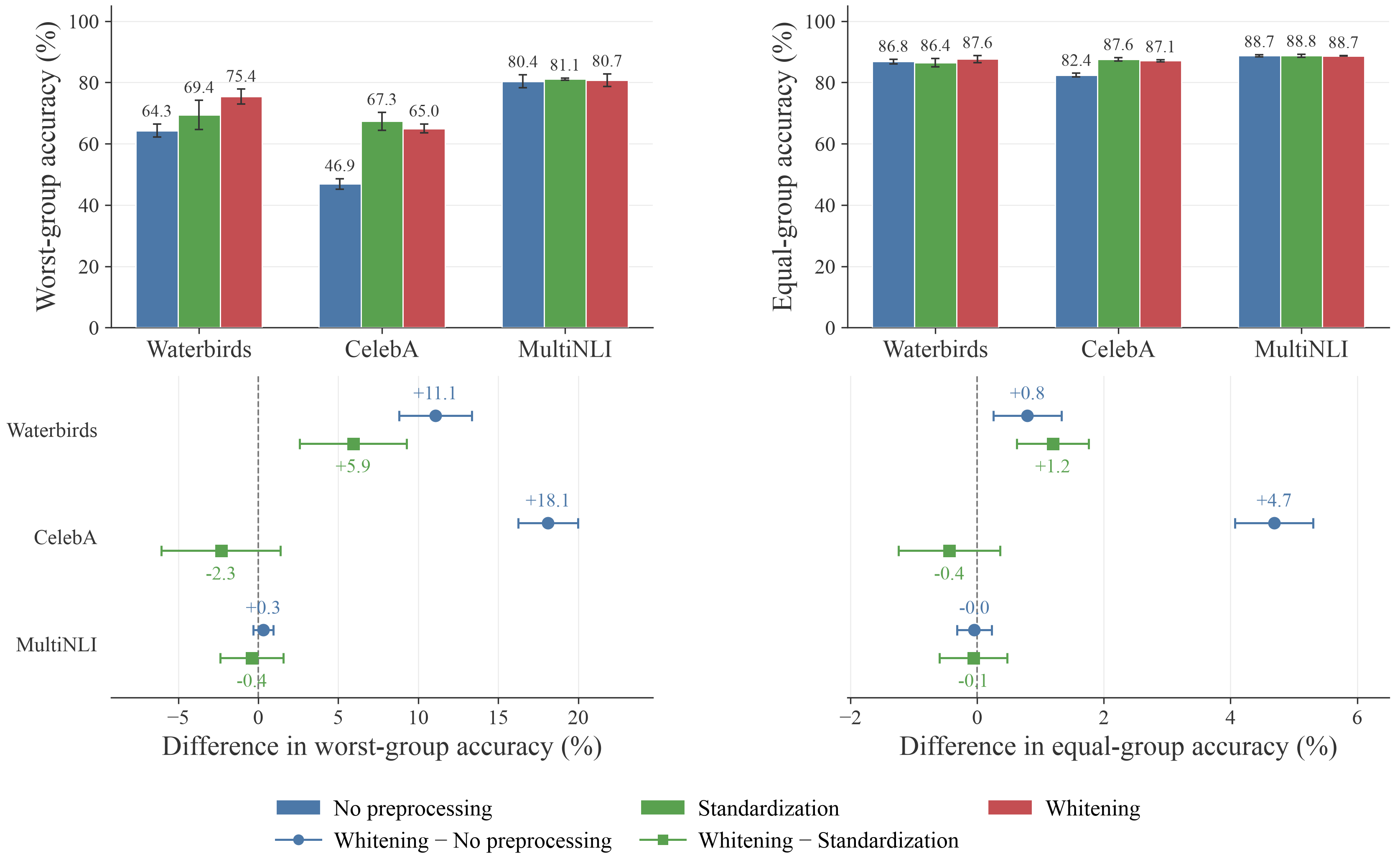}
    \caption{Worst-group (left) and equal-group (right) test accuracy for DINO ViT-B/16 on Waterbirds and CelebA and DeBERTa-v3-base on MultiNLI under three preprocessing variants: no preprocessing, standardization, and whitening (top), together with mean within-seed paired differences for whitening minus each reference (bottom). Bars and points represent averages across three seeds, and error bars are 95\% confidence intervals.}
    \label{fig:empirical-appendix-backbones}
\end{figure}

Figure~\ref{fig:empirical-comparison-equal-group} is the equivalent of Figure~\ref{fig:empirical-comparison} but for equal-group accuracy.
We observe similar results, with whitening improving equal-group accuracy for DFR and AFR on Waterbirds. We do observe a minor decrease in equal-group accuracy when AFR is applied to CelebA ($-0.8$ pp vs. no preprocessing, and $-0.6$ pp vs. standardization), as well as for NT on CelebA ($-0.8$ pp vs. standardization).

\begin{figure}[h]
    \centering
    \includegraphics[width=\textwidth]{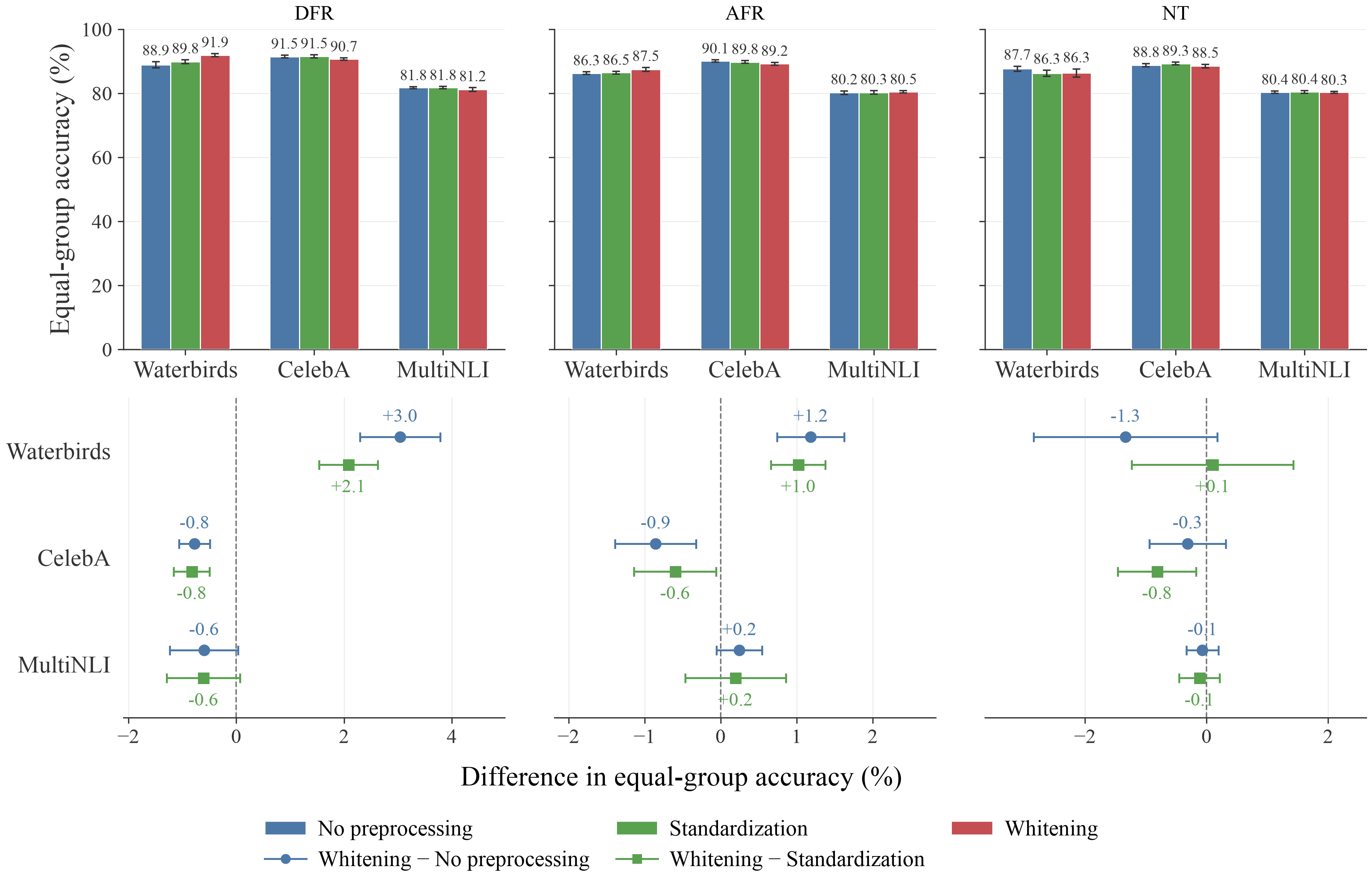}
    \caption{Equal-group test accuracy for Deep Feature Reweighting (DFR), Automatic Feature Reweighting (AFR), and NeuronTune (NT) under three preprocessing variants: no preprocessing, standardization, and whitening (top), together with mean within-seed paired differences for whitening minus no preprocessing and whitening minus standardization (bottom). Bars and points show averages across seeds (ten for ResNet-50, five for BERT), and error bars show 95\% confidence intervals.}
    \label{fig:empirical-comparison-equal-group}
\end{figure}

\clearpage
\section{Details on empirical evaluation}
\label{app:empirical-evaluation}

\subsection{Datasets}
\label{app:datasets}

\textbf{Waterbirds} \citep{sagawa2020distributionallyrobustneuralnetworks}. is a binary
classification task in which the target is whether an
image contains a waterbird or a landbird. The background---water or
land---is spuriously correlated with the target in the training and validation
sets. We therefore define four groups from the bird type and background type.

\textbf{CelebA} \citep{liu2015faceattributes}, the target is whether a person has
blond hair and the spurious
attribute is gender. Blond hair is substantially more frequent among women in
the training and validation sets, yielding four groups defined by hair color
and gender.

\textbf{MultiNLI} \citep{Bowman_MultiNLI} is a three-class natural-language inference
task in which the target
indicates whether a hypothesis is entailed by, neutral with respect to, or
contradicted by a premise. Following the standard spurious correlation setup,
we use the presence of a negation word in the hypothesis as the spurious
attribute; negation words are disproportionately associated with contradiction
examples. The combination of the three target labels and this binary attribute
defines six groups.

\begin{table}[H]
     \caption{Training and validation split sizes for Waterbirds, CelebA, and MultiNLI.
    The table reports total observations and counts by class $y$ and group $g$.
    For Waterbirds and CelebA, groups are indexed by $g=2y+a+1$, where $a$ is
    the binary spurious attribute. For MultiNLI, groups are indexed by
    $g=2y+a$.}
    \setlength{\tabcolsep}{3pt}
    \begin{tabular}{>{\raggedright\arraybackslash}p{1.9cm} >{\raggedright\arraybackslash}p{3.6cm} >{\raggedright\arraybackslash}p{3.6cm} >{\raggedright\arraybackslash}p{3.9cm}}
        \toprule
        \textbf{Statistic} & \textbf{Waterbirds} & \textbf{CelebA} & \textbf{MultiNLI} \\
        \hline
        Train total & 4,795 & 162,770 & 206,175 \\
        \hline
        Train per class & \makecell[l]{$y=0$: 3,670\\$y=1$: 1,125} & \makecell[l]{$y=0$: 138,503\\$y=1$: 24,267} & \makecell[l]{$y=0$: 68,656\\$y=1$: 68,897\\$y=2$: 68,622} \\
        \hline
        Train per group & \makecell[l]{$g=1$: 3,486\\$g=2$: 184\\$g=3$: 56\\$g=4$: 1,069} & \makecell[l]{$g=1$: 66,874\\$g=2$: 71,629\\$g=3$: 1,387\\$g=4$: 22,880} & \makecell[l]{$g=0$: 57,498\\$g=1$: 11,158\\$g=2$: 67,376\\$g=3$: 1,521\\$g=4$: 66,630\\$g=5$: 1,992} \\
        \midrule
        Val.\ total & 1,199 & 19,867 & 82,462 \\
        \hline
        Val.\ per class & \makecell[l]{$y=0$: 945\\$y=1$: 254} & \makecell[l]{$y=0$: 16,811\\$y=1$: 3,056} & \makecell[l]{$y=0$: 27,448\\$y=1$: 27,562\\$y=2$: 27,452} \\
        \hline
        Val.\ per group & \makecell[l]{$g=1$: 898\\$g=2$: 47\\$g=3$: 13\\$g=4$: 241} & \makecell[l]{$g=1$: 8,276\\$g=2$: 8,535\\$g=3$: 182\\$g=4$: 2,874} & \makecell[l]{$g=0$: 22,814\\$g=1$: 4,634\\$g=2$: 26,949\\$g=3$: 613\\$g=4$: 26,655\\$g=5$: 797} \\
        \bottomrule
    \end{tabular}
    \label{tab:dataset_counts}
\end{table}

\subsection{Evaluation protocol}
\label{app:protocol}

\paragraph{Waterbirds.}
We use a ResNet-50 backbone trained for 50 epochs with SGD, momentum $0.9$, a
learning rate of $0.001$, weight decay $0.01$, and batch size 32. We use
standard random-crop and horizontal-flip augmentation, without early stopping
or validation-based checkpoint selection. We additionally fine-tune a
ViT-B/16 initialized from DINO self-supervised pretraining
\citep{caron2021dino} for 120 epochs with AdamW, $\beta_1=0.9$,
$\beta_2=0.95$, a learning rate of $10^{-5}$, weight decay $0.01$, batch size
32, and a cosine learning-rate schedule. For ResNet-50, we use the feature
vector preceding the classification layer; for ViT-B/16, we use the final
classification-token representation.

\paragraph{CelebA.}
We fine-tune an ImageNet-pretrained ResNet-50 for 50 epochs with SGD, momentum
$0.9$, a learning rate of $0.001$, weight decay $0.0001$, and batch size 128.
We use the same crop and flip augmentation as for Waterbirds, without early
stopping or validation-based checkpoint selection. We additionally fine-tune
the DINO-pretrained ViT-B/16 for 60 epochs with AdamW, $\beta_1=0.9$,
$\beta_2=0.95$, a learning rate of $10^{-5}$, weight decay $0.01$, batch size
128, and a cosine learning-rate schedule. We extract the same ResNet-50 and
ViT-B/16 representations as for Waterbirds.

\paragraph{MultiNLI.}
For the main comparison, we fine-tune BERT for five epochs with AdamW, a
learning rate of $10^{-5}$, weight decay $10^{-4}$, and batch size 16. We
additionally fine-tune DeBERTa-v3-base \citep{he2023debertav3} for three epochs
with AdamW, $\beta_1=0.9$, $\beta_2=0.95$, a learning rate of
$1.5\times10^{-5}$, weight decay $0.01$, batch size 64, and a cosine
learning-rate schedule with 1,000 warm-up steps. For DeBERTa-v3-base, we use
dropout $0.1$ and truncate inputs to 128 tokens. For both backbones, we extract
the first-token representation from the final hidden layer.

\paragraph{Last-layer estimation.} We fit preprocessing on the training
representations and apply it unchanged to validation and test. Given the
resulting embeddings
$\{\boldz_i\}_{i=1}^n$, we fit an $\ell_2$-regularized
\texttt{LogisticRegression} estimator from \texttt{scikit-learn}. For the
binary image tasks, let $n_y$ denote the number of training observations with
label $y$, and assign observation $i$ the inverse-class-frequency weight
$\omega_i=n/(2n_{\target_i})$. These weights have mean one, and the estimator
minimizes
\begin{align}
    \frac1n\sum_{i=1}^n \omega_i
    \log\!\left(
        1+\exp\!\left(-\target_i
        (\boldw^\top\boldz_i+b)\right)
    \right)
    +\frac{\tau}{2}\|\boldw\|_2^2.
    \label{eq:empirical-logistic-objective}
\end{align}
The intercept $b$ is not penalized. This weighting makes the two labels
contribute equally without using group labels. For MultiNLI, whose three
training-label counts are nearly equal, we fit the ordinary uniformly weighted
multinomial objective. Every reported fit is required to converge.

\paragraph{Regularization and model selection.}
We consider the 24-value grid for the $\ell_2$ penalty $\tau$:
\begin{align}
    \tau \in \{1,5\}\times 10^k,
    \qquad k\in\{-5,-4,\ldots,6\}.
\end{align}
Here, $\tau$ is defined relative to the average logistic loss in
\eqref{eq:empirical-logistic-objective}. The \texttt{scikit-learn} objective
uses the sum of the observation losses plus a regularization term with
coefficient $1/(2C)$. Dividing this objective by the training set size $n$
gives an average loss with regularization coefficient $1/(2nC)$. We therefore
set $C=(n\tau)^{-1}$ to account for the dataset size.

For each preprocessing method and representation seed, we select $\tau$ by
maximizing class-balanced validation accuracy. Let $\mathcal{Y}$ be the set of
$K$ target classes, let $\mathcal{V}_k=\{i\in\mathcal{V}:y_i=k\}$ denote the
validation observations in class $k$, and let $\hat y_i$ be the predicted
label. We define
\begin{align}
    \operatorname{Acc}_{\mathrm{bal}}
    = \frac{1}{K}\sum_{k\in\mathcal{Y}}
      \frac{1}{|\mathcal{V}_k|}
      \sum_{i\in\mathcal{V}_k}\mathbf{1}\{\hat y_i=y_i\}.
\end{align}
Thus, each class contributes equally, with each observation weighted
inversely by its class count. We resolve exact ties in favor of the larger
$\tau$.

\paragraph{Details on whitening.} The $\hat\delta_i$ terms introduced in Section~\ref{sec:est} are defined by \citet{ledoit2020analytical} in 
equation 4.3. Importantly, these adjustments are defined for when $n - 1 > p$. When $p>n-1$, we operate on lower-dimensional sample space and work with the $n\times n$ Gram matrix instead of the $p\times p$ empirical covariance matrix. 
We compute the shrinkage adjustments using this lower-dimensional representation and apply them to the nonzero eigenvalues of the empirical covariance matrix.

\subsection{Other approaches}
\label{app:other-approaches}

\textbf{Deep Feature Reweighting (DFR)} \citep{kirichenko2023layerretrainingsufficientrobustness}.
We divide the validation set into two group-stratified halves. From the first
half, we construct ten group-balanced samples by drawing the same number of
observations from every group. For each preprocessing variant and value of
$\tau$, we fit the preprocessing transform and linear probe separately on
each balanced sample. We map the resulting coefficients and intercepts to the
original representation coordinates and average them, which is equivalent to
averaging the logits of the ten probes. We select $\tau$ using worst-group
accuracy on the second validation half. After selection, we repeat this
group-balanced ten-sample procedure using the complete validation set as the
sampling pool. Thus, the final fits draw from the complete validation set, but
each individual fit remains group-balanced. The original DFR protocol uses
$\ell_1$-regularized logistic regression, whereas
\citet{labonte2023lastlayerretraininggrouprobustness} use $\ell_2$
regularization for last-layer retraining. We follow the latter approach.

\textbf{Automatic Feature Reweighting (AFR)} \citep{qiu2023simplefastgrouprobustness}.
We divide the validation set into two label-stratified halves, which we use for
adaptation and model selection, respectively. On the adaptation half, we use
the untransformed fine-tuned DNN to calculate the correct-class probability
$p_i$. If $n_y$ denotes the number of adaptation observations with label $y$,
we assign observation $i$ the unnormalized weight
$\exp(-\gamma p_i)/n_{y_i}$ and normalize the weights to have mean one. We
consider
\begin{align}
    \gamma \in \{1,2,4,8\}.
\end{align}
We fit the preprocessing transform on the adaptation representations and
minimize the weighted logistic loss with an $\ell_2$ penalty centered at the
original ERM head, expressed in the transformed coordinates. Thus, the usual
zero-centered $\ell_2$ penalty is replaced by
$\tau\|\boldw-\boldw_{\mathrm{ERM}}\|_2^2/2$. We jointly select $\gamma$
and $\tau$ using worst-group accuracy on the selection half, and
subsequently recompute the transform, weights, and classifier on the complete
validation set. Unlike the original AFR implementation, we optimize each
candidate to convergence rather than using early stopping.

\textbf{NeuronTune (NT)} \citep{pmlr-v267-zheng25s}.
We divide the validation set into two label-stratified halves. Let $h_{ij}$
denote coordinate $j$ of the raw representation of observation $i$, and let
$\hat y_i$ denote the prediction of the fine-tuned DNN. On the first half, we
calculate the score
\begin{align}
    s_{yj}
    ={}&
    \underset{i:y_i=y,\,\hat y_i\neq y_i}{\operatorname{median}}
    |h_{ij}|
    -
    \underset{i:y_i=y,\,\hat y_i=y_i}{\operatorname{median}}
    |h_{ij}|.
\end{align}
We remove coordinate $j$ when $s_{yj}>\delta$ for at least one class $y$, and
consider
\begin{align}
    \delta \in \{0,0.1,0.25,0.5\}.
\end{align}
For each value of $\delta$, the resulting coordinate set is held fixed across
preprocessing variants and values of $\tau$. We apply standardization or
whitening after coordinate removal, and fit the preprocessing transform and
linear probe on the second validation half. We jointly select $\delta$ and
$\tau$ using exact class-balanced accuracy on the first half. The selected
second-half fit is then evaluated on the test set without a complete-validation
refit. The original method iterates coordinate identification and gradient descent, and uses a custom criterion for model selection. 
We use a single identification step, as we use the standard optimization techniques from \texttt{scikit-learn}.
We find that selecting based on class-balanced accuracy performs better than the custom criterion introduced by \citet{pmlr-v267-zheng25s}.
While we are aware that this is a slightly different implementation of the method, we are primarily interested in isolating the effect of different preprocessing techniques, rather than optimizing worst-group accuracy.

We use the same ridge grid and estimator settings as described in Appendix~\ref{app:protocol}.

\newpage
\bibliographystyle{iclr2027_conference}
\bibliography{bibliography}

\end{document}